\documentclass[journal]{IEEEtran}
\usepackage{mydefinitions}

\begin{document}

\title{BGTplanner: Maximizing Training Accuracy for Differentially Private Federated Recommenders via Strategic Privacy Budget Allocation}

\author{Xianzhi Zhang,
        Yipeng Zhou,~\IEEEmembership{Member,~IEEE,}
        Miao Hu,~\IEEEmembership{Member,~IEEE,}
        Di Wu,~\IEEEmembership{Senior Member,~IEEE},
        Pengshan Liao,
        Mohsen Guizani,~\IEEEmembership{Fellow,~IEEE,}  
        Michael Sheng,~\IEEEmembership{Member,~IEEE}
\IEEEcompsocitemizethanks{
\IEEEcompsocthanksitem Xianzhi Zhang, Miao Hu, Di Wu, and Pengshan Liao are with the School of Computer Science and Engineering, Sun Yat-sen University, Guangzhou, China, and Guangdong Key Laboratory of Big Data Analysis and Processing, Guangzhou, China. (E-mail: \{zhangxzh9, liaopsh\}@mail2.sysu.edu.cn; \{humiao5, wudi27\}@mail.sysu.edu.cn). 
\IEEEcompsocthanksitem Yipeng Zhou and Michael Sheng are with the Department of Computing, Faculty of Science and Engineering, Macquarie University, Sydney, Australia. (E-mail: \{yipeng.zhou, michael.sheng\}@mq.edu.au,).
\IEEEcompsocthanksitem Mohsen Guizani is with the Machine Learning Department, Mohamed Bin Zayed University of Artificial Intelligence, Abu Dhabi, UAE (E-mail: mguizani@ieee.org).
}
}



\maketitle

\begin{abstract}
To mitigate the rising concern about privacy leakage, the federated recommender (FR) paradigm emerges, in which decentralized clients co-train the recommendation model without exposing their raw user-item rating data. The differentially private federated recommender (DPFR) further enhances FR by injecting differentially private (DP) noises into clients. 
Yet, current DPFRs, suffering from noise distortion, cannot achieve satisfactory accuracy. Various efforts have been dedicated to improving DPFRs by adaptively allocating the privacy budget over the learning process. However, due to the intricate relation between privacy budget allocation and model accuracy, existing works are still far from maximizing DPFR accuracy. 
To address this challenge, 
we develop BGTplanner (Budget Planner) to strategically allocate the privacy budget for each round of DPFR training, improving overall training performance.
Specifically, we leverage the Gaussian process regression and historical information to predict the change in recommendation accuracy with a certain allocated privacy budget.  
Additionally, Contextual Multi-Armed Bandit (CMAB) is harnessed to make privacy budget allocation decisions by reconciling the current improvement and long-term privacy constraints. 
Our extensive experimental results on real datasets demonstrate that \emph{BGTplanner} achieves an average improvement of 6.76\% in training performance compared to state-of-the-art baselines.
\end{abstract}


\section{Introduction}
Nowadays, machine learning has been widely applied in various recommender systems, necessitating the processing of substantial user data. Consequently, there is increasing concern from users about the security and privacy of their data~\cite{li2023optimal,Hu2022}.
To protect private raw data, the federated recommender (FR) paradigm \cite{wu2021fedgnn,Guo2022,Wu2024} has emerged, offering a solution to the conundrum of balancing data-driven progress with the preservation of individual privacy. 
In this evolving landscape, organizations 
mainly exchange training parameters with distributed data sources for improving recommendations while respecting user privacy~\cite{Agrawal_Sirohi_Kumar_2024}. 

Nevertheless, revealing training parameters as opposed to raw data proves inadequate for ensuring privacy protection, invoking the development of differentially private federated recommenders (DPFRs)~\cite{Agrawal_Sirohi_Kumar_2024,Jiang2023}.
Studies indicate that training parameters, including gradient values, can facilitate the reconstruction of a segment of original data \cite{DBLP:conf/sp/MelisSCS19}, or enable inference about whether the records pertained by a specific participant \cite{DBLP:conf/kdd/LiuCWLCC23,ArroyoArevalo_Noorbakhsh_Dong_Hong_Wang_2024}. 
To ensure data privacy, the additional safeguard, \emph{i.e.}, differential privacy (DP) \cite{cormode2018privacy,Kean2024,He2024}, is involved by DPFRs, which inject noises to distort training parameters, 
and have emerged as a key consideration in privacy-preserving recommender systems. 

However, the practical application of DPFRs is impeded by substantially compromised recommendation accuracy due to the distortion caused by DP noise. 
The challenge is rooted in the complicated relationship between privacy budget consumption and learning improvement over multiple times of training parameter exposure in the DPFR learning process. 
On the one hand,  to guarantee privacy, the total privacy budget is constrained, and only a fixed small amount of privacy budget can be consumed by each exposure~\cite{Kang2020, Girgis2021}. Whereas, a smaller privacy budget implies a larger noise scale and more significant detriment to DPFR accuracy. On the other hand, the FR training process is highly dynamic with heterogeneous data scattered among users, resulting in a dynamic learning process~\cite{Zhang_Chen_Ma_Fang_King_2024}. Consequently, it is difficult to predict the improvement of each learning round in advance.

When facing multiple times of exposures, the naive approach evenly allocates a fixed amount of privacy budget to each training round regardless of discrepant learning improvements during the dynamic training process. Such an approach fails to maximize recommendation accuracy because the privacy budget is probably wasted if it is spent on protecting users bringing little learning improvement. 
Recently, a few works have attempted to tune the privacy budget allocation over multiple learning rounds adaptively. However, most of these works such as  ADPML~\cite{electronics12030658} are heuristics-based algorithms, far away from the optimal budget allocation.

\begin{table*}[ht]
\centering
\caption{The main notations and definitions in the paper.}
\label{Table:Major Notations}
\renewcommand{\arraystretch}{1.2}
\rowcolors{2}{white}{gray!25} 
\begin{tabular}{p{2.7cm}<{\centering} m{\linewidth-3.55cm}}
    \toprule 
        Notation  & \multicolumn{1}{c}{Definition}  \\
    \midrule
        $u\in \mathcal{U}$ & The client index and the set of all clients.\\
        $ t \ /\ \tau \ /\ T$ & The training round index, the total training rounds with early termination and the maximum training rounds.\\
        $\bm{\mathrm{X}}^{t}\ /\ \bm{x}^{t}$ & The original context matrix of training data information and the SVD-reduced vector in round $t$. \\
        $d$ & The hyper-parameter representing the dimension of the vector $\bm{x}^{t}$. \\
        $\bm{\epsilon}^{total}\ /\ \bm{\epsilon}^{t}$ & The total privacy budget vector and the privacy budget cost vector for all clients in training round $t$. \\
        $(\epsilon_u^{t}, \delta_u^t)$ & The privacy budget of client $u \in\mathcal{U}$ in round $t$.\\
        ${C}_u(\cdot)$ & The budget generating function of DP noise on client $u$. \\
        $\bm{\iota}_u^{t}$\ /\  $\bm{\nu}_u^{t}$ \ /\  $ \bm{\omega}^{t}_u$ & The parameters of the local item embedding, the local user embedding and recommender model for clients $u$ in round $t$.\\
        $\nabla \bm{\iota}_u^{t}$\ /\  $\nabla \bm{\nu}_u^{t}$ \ /\  $\nabla \bm{\omega}^{t}_u$ & The gradients of the local item embedding, the local user embedding and recommender model for clients $u$ in round $t$.\\
        $a^t\in\mathcal{A}$ & The budget allocation action in round $t$ and the action space.\\
        $A$     & The size of action space.\\
        $r^{t}$ \ /\ $\bm{r}^{t}$ & The reward of model training in $t$ round, 
         which can be defined as the change of training loss or recommender accuracy between round $t$ and $t-1$. \ /\ The vector of all observed rewards until round $t$.\\
        $\bm{z}^t$ & The concatenated input vector of action $a\in\mathcal{A}$ and context vector $\bm{x}^t$ for GPR model.\\
        $\hat{r}(\bm{z}^{t})$ or $\hat{r}(a,\bm{x}^t)$ & The predicted reward by GPR model with action $a\in\mathcal{A}$ and context vector $\bm{x}^t$.\\

        ${\mu}(\bm{z}^{t})$ \ /\ ${\sigma}(\bm{z}^{t})$ &  The mean function and variance function of the conditional distribution $\mathcal{N}({\mu}(\bm{z}^{t}), \sigma(\bm{z}^{t}))$ of GPR model to predict reward $\hat{r}(a,\bm{x}^t)$.\\
        $\bm{\lambda}^t$ & The vector of dual variables (Lagrange multipliers) for long-term budget constraint in round $t$.\\
        $\mathcal{V}$ \ /\ $\Lambda$ & The space of dual variables $\bm{\lambda}$ and the $l_1$-radius of the space $\mathcal{V}$.\\
        $T_0$ &  The total rounds for playing each action in the initial stage of BGTplanner.\\
        $\beta^t (a)$ & The CMAB score function for action $a$ in round $t$.\\
        $p^t (a)$ & The probability for CMAB to choose action $a$ in round $t$.\\
    \bottomrule
    \end{tabular}
    \vspace{-2mm}
\end{table*} 

To address the challenges in allocating privacy budgets in DPFR, we introduce a novel online  algorithm, called \emph{BGTplanner} (Budget Planner). Different from prior approaches, our method employs the Gaussian process regression (GPR)  to predict learning improvement on the fly based on historical information and dynamic contexts (e.g., dynamic information of training items). Next, Contextual Multi-Armed Bandit (CMAB) is leveraged to balance the tradeoff between the current learning improvement and long-term privacy budget constraints, which can avoid greedily depleting the privacy budget in a single round. 
Our contributions can be summarized as follows:
\begin{itemize}
     \item  We formulate an online privacy budget allocation problem for the DPFR system. 
     Then, we exploit historical information and  Gaussian process regression to predict the dynamic learning improvement under different contexts. 

    \item  We harness the Contextual Multi-Armed Bandit (CMAB) method to make privacy budget allocation decisions for reconciling the current learning improvement and long-term privacy budget constraints. 
    \item  We provide sufficient theoretical analyses for BGTplanner. Furthermore, through extensive experiments, we demonstrate the superiority of our method compared to state-of-the-art baselines, with an average improvement of 6.76\% in RMSE and 0.71\% in F1-score. 
\end{itemize}

{\color{black}
The remainder of this paper is organized as follows. In Section \ref{sec:related work}, we review the related works. Section \ref{sec:preliminaries} provides the necessary background on DPFR and CMAB. Our system model and problem formulation are discussed in Section \ref{sec:problemformulation}. Section \ref{sec:algorithm design} introduces the design of our algorithm, along with an analysis of differential privacy and regret. In Section \ref{sec:perfomance}, we present comprehensive performance evaluations. Finally, we conclude the paper in Section \ref{sec:conclusion}.
}

\section{Related Work}\label{sec:related work}

\subsection{Training Federated Recommender with DP}
Training recommender models often relies on sensitive data, such as users' historical interactions with items. To protect origin data from leaking to malicious attackers, FR, a distributed learning framework~\cite{mcmahan2017} has been widely adopted for training recommender models~\cite{wu2021fedgnn,Agrawal_Sirohi_Kumar_2024}. 
Nevertheless, revealing training parameters rather than raw data has proven inadequate for ensuring privacy protection~\cite{DBLP:conf/sp/MelisSCS19,
DBLP:conf/kdd/LiuCWLCC23, ArroyoArevalo_Noorbakhsh_Dong_Hong_Wang_2024}, leading to the development of DPFR. 
To safeguard data privacy, DPFRs incorporate 
DP mechanisms~\cite{cormode2018privacy}, which enhance privacy by injecting noises to distort exposed training parameters. This approach has emerged as a key consideration in privacy-preserving recommenders~\cite{chen2022differential}, despite that its adoption in practice is still hindered by lowered recommendation accuracy.

\subsection{Differential Privacy Budget Allocation}
The privacy budget is a non-recoverable resource, which must be split into multiple small portions to restrict the privacy loss per training round~\cite{Kang2020, Girgis2021}.
It has been reported in \cite{chen2022differential} that DP noises generated by a small privacy budget per round can severely impair model accuracy. 
Existing studies have initially attempted to address the influence of noise on DPFR. 
For example, \cite{DBLP:conf/fat/HongWZ22} explored dynamic techniques for adjusting the learning rate, batch size, noise magnitude, and gradient clipping to improve DPFR accuracy. Yet, this work does not account for dynamic factors in the recommender training process, such as online interactive ratings and randomly generated noises~\cite{wu2021fedgnn}. 
The dynamic strategy introduced by \cite{electronics12030658} focused on privacy budget allocation in a heuristic-based method. 
However, when applied to DPFR, it confronts the limitation that a heuristically designed strategy can be far from the optimal privacy budget allocation. 

CMBA is an online decision-making scheme applicable in complex dynamic environments~\cite{DBLP:journals/corr/abs-2303-10373}. 
\cite{pmlr-v130-shi21c} applied CMBA to improve personalized FL. \cite{DBLP:journals/iotj/WuZZCWC23, DBLP:journals/corr/abs-2303-10373} applied CMBA to 
improve client selection for FL framework 
by overcoming dynamic training environments. Because of the dynamic nature of user interactions with recommenders, we are among the first to utilize CMAB for optimizing the allocation of the privacy budget in DPFR.

\section{Preliminaries}\label{sec:preliminaries}

\subsection{Differentially Privacy and Federated Recommender}
{\color{black}
Differential privacy (DP)~\cite{Xia2024,Kean2024} is a mechanism designed to limit an attacker's ability to gain additional knowledge about individuals in a dataset. A typical approach to achieve DP is by introducing randomness into computations, which conceals the details of individual data records. Formally, DP is characterized by $\epsilon$, which defines the upper bound on the difference between two neighboring inputs, and $\delta$, which represents the residual probability. The formal definition of classic ($\epsilon, \delta$)-DP is provided in Definition~\ref{definition:DP}.

\begin{definition}[($\epsilon, \delta$)-DP]
For any $\epsilon > 0$ and $\delta \in [0, 1]$, a randomized algorithm $\mathcal{M}$: $\mathcal{D} \rightarrow \mathcal{O}$ satisfies ($\epsilon, \delta$)-DP if for any two neighboring inputs $D, D^{'} \in \mathcal{D}$ and any subset of output $O \subset \mathcal{O}$, the following holds:
\begin{equation}
\mathbf{Pr}[\mathcal{M}(D) \in O] \leq e^{\epsilon} \cdot \mathbf{Pr}[\mathcal{M}(D^{'}) \in O] + \delta.
\end{equation}
\label{definition:DP}
\end{definition}
In the DP mechanism, revisiting the same dataset increases the risk of privacy leakage compared to a single access~\cite{Xiao2024}. The composition theorem~\cite{dwork2014algorithmic} states that the privacy budget per query should decrease as the number of accesses to the dataset grows. The most fundamental form of this is the sequential composition theorem, presented in Theorem~\ref{theorem:Sequential_Composition}.

\begin{theorem}[Sequential Composition~\cite{dwork2014algorithmic}]
Let the randomized algorithm $\mathcal{M}^{t}$: $\mathcal{D} \rightarrow \mathcal{O}$ satisfy ($\epsilon^{t}, \delta^{t}$)-DP for all $t \leq T, t\in \mathbb{N}^{+}$. Then, the composition $\mathcal{M}$ of all algorithms for the dataset $D \in \mathcal{D}$, defined by $\mathcal{M}(D) = (\mathcal{M}_{1}(D), \dots, \mathcal{M}_{T}(D))$, satisfies ($\sum_{t=1}^{T} \epsilon^{t}, \sum_{t=1}^{T} \delta^{t}$)-DP.
\label{theorem:Sequential_Composition}
\end{theorem}

Following this, in federated learning~\cite{mcmahan2017,Zhou2024}, clients can securely update the recommender models using private datasets and upload the noisy gradients to the server for global aggregation in each training round. Specifically, a client $u$ in the differentially private federated recommender (DPFR) framework updates the gradients for the local item embedding $\nabla \bm{\iota}_u^{t}$, the local user embedding $\nabla \bm{\nu}_u^{t}$, and other recommender parameters $\nabla \bm{\omega}^{t}_u$, respectively~\cite{wu2021fedgnn,qi2020fedrec}. Before uploading these local recommender models to the server, the client distorts the local gradients by adding noise: $\nabla\tilde{\bm{\omega}}^t_u = \nabla \bm{\omega}^{t}_u + \bm{n}_\omega$ and $\nabla \tilde{\bm{\iota}}_u^{t} = \nabla \bm{\iota}_u^{t} + \bm{n}_\iota$, where $\bm{n}_u = [\bm{n}_\omega, \bm{n}_\iota]$ represents the DP noise generated by a specific DP noise mechanism (e.g., Gaussian~\cite{Kang2020} or Laplace~\cite{Girgis2021}) based on the allocated privacy budget $(\epsilon_u^t, \delta_u^t)$. \footnote{The user embedding is generally stored locally after being updated, and therefore does not require noise addition.}

\subsection{Contextual Multi-Arm Bandit Algorithm}
The foundational concept of the Contextual Multi-Armed Bandit (CMAB) algorithm is essential for addressing decision-making challenges that involve balancing exploration and exploitation. Analogous to rows of slot machines (one-armed bandits), the CMAB algorithm deals with selecting actions over successive rounds to maximize cumulative rewards. Let $A$ denote the number of arms (or actions), and $t$ represent discrete time steps. At each time step $t \in \{1,\dots, T\}$, the algorithm selects an action $a^{t}$ from the space of action $\mathcal{A}=\{1,2,\dots,A\}$ based on observed environmental information (e.g., any context vector matrix $\mathbf{X}^t$). The reward $r^{t}$ corresponding to the chosen action is then observed. The goal of the algorithm is to iteratively refine its action-selection strategy to maximize the expected cumulative reward over time, which can be expressed as:
\begin{equation}
\max_{\bm{a}} \sum_{t=1}^T \mathbb{E}\left[r^{t} \mid a^{t},\mathbf{X}^{t}\right].
\label{eq: max cumulative reward}
\end{equation}

}




\section{SYSTEM MODEL}\label{sec:problemformulation}
This section encompasses preliminary knowledge of DPFR, problem formulation, and BGTplanner overview.
To facilitate the readability, we have compiled a summary of notations used in this paper in Table~\ref{Table:Major Notations}.



\subsection{System Overview}
In Fig.~\ref{fig:system}, we present the workflow of a typical DPFR system. There are two main components: the server and clients. The server plays two functions: aggregating model parameters of the recommender and allocating the privacy budget to each client to determine the noise 
scale~\cite{DBLP:conf/sp/NasrSH19, DBLP:conf/sp/ShejwalkarHKR22}. 
Clients are responsible for conducting local model updates with private datasets. Before uploading local recommender models to the server, a DP 
noise 
adder will generate noises according to the allocated privacy budget to distort 
local models~\cite{wu2021fedgnn, chen2022differential,Agrawal_Sirohi_Kumar_2024}.

In the system, there are $U$ clients, denoted by the set $\mathcal{U} = \{1,\cdots,U\}$. The DPFR can conduct the maximum $T$ rounds of model training.
The total privacy loss of each client is constrained by a fixed total privacy budget $(\epsilon_u^{total}, \delta_u^{total}), \forall u\in\mathcal{U}$ to be consumed throughout the entire training. 
During the training process, clients and the server need to exchange information 
multiple times via the following five steps. 

\begin{figure}[t]      
\centering      
\includegraphics[width=\linewidth]{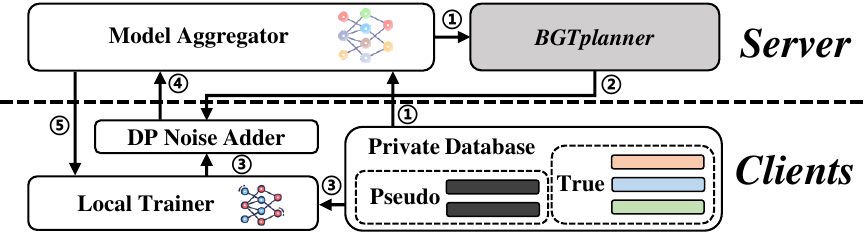}   
\caption{An overview of a typical DPFR system.}
\label{fig:system}  
\vspace{-4mm}
\end{figure}

In \textbf{Step \textcircled{1}},  all clients upload ID numbers of item samples used for 
the 
local training. The server can reform 
the ID information as the context matrix denoted by $\bm{\mathrm{X}}^{t} = [\mathrm{X}_{u,i}^t]^{U\times I}$. 
Here, $I$ represents the total number of items, and $\mathrm{X}_{u,i}^t\in\{0,1\}$.  $\mathrm{X}_{u,i}^t=1$ indicates that item $i$ is used for training on client $u$ in this round, while 
$\mathrm{X}_{u,i}^t=0$ signifies that the item is not used.
Note that $\bm{\mathrm{X}}^{t}$ is commonly uploaded together with updated parameters in 
Step 
\textcircled{4} for aggregating item embeddings in related works~\cite{wu2021fedgnn}. 
It 
makes no difference 
in 
uploading 
$\bm{\mathrm{X}}^{t}$ in either 
step\footnote{It is worth mentioning that $\bm{\mathrm{X}}^{t}$ before uploading can be obfuscated by clients with pseudo records to enhance privacy. Since our focus is not on how to obfuscate  $\bm{\mathrm{X}}^{t}$, we simply assume that existing obfuscation methods ~\cite{wu2021fedgnn} will be applied.}. 

In \textbf{Step \textcircled{2}}, the server allocates $(\epsilon_u^{t}, \delta_u^t)$ privacy budget to client $u \in\mathcal{U}$ in round $t$. 
If a client exhausts its privacy budget, it quits the training process. 
Besides, the server can verify the noise scale~\cite{DBLP:conf/sp/NasrSH19,DBLP:conf/sp/ShejwalkarHKR22} and account cumulative privacy budget consumption using privacy accounting methods for different DP mechanisms, such as native composition for Laplace noises~\cite{Kang2020} and Moments Accountant for Gaussian noises~\cite{DBLP:conf/ccs/AbadiCGMMT016}.

In \textbf{Step \textcircled{3}} and \textbf{\textcircled{4}}, client $u$ updates the recommender model via  \emph{Local Trainer} using the latest user-item interactions in local data. 
Local updating yields  gradients of the local item embedding $\nabla \bm{\iota}_u^{t}$, the local user embedding $\nabla \bm{\nu}_u^{t}$ and other recommender parameters $\nabla \bm{\omega}^{t}_u$, respectively~\cite{wu2021fedgnn}. 
After local training, client $u$ uploads $ \nabla\tilde{\bm{\omega}}^t_u = \nabla \bm{\omega}^{t}_u+\bm{n}_\omega $ and $\nabla \tilde{\bm{\iota}}_u^{t} = \nabla \bm{\iota}_u^{t} + \bm{n}_\iota$ where  $\bm{n}_u = [\bm{n}_\omega,\bm{n}_\iota]$ represents DP noises generated by \emph{DP Noise Adder} according to the allocated privacy budget $(\epsilon_u^t, \delta_u^t)$. 

In \textbf{Step \textcircled{5}}, \emph{Model Aggregator} deployed at the server aggregates recommender parameters and item embedding vectors based on $\bm{\mathrm{X}}^{t}$. For more details, please refer to~\cite{wu2021fedgnn}.
Then, the server broadcasts aggregated parameters and item embedding vectors to all clients.

\subsection{Online Privacy Budget Allocation Problem}
Our study primarily focuses on designing a more advanced BGTplanner.
In our design, we suppose that the server can select an action $a^{t} \in \mathcal{A}$ for privacy budget allocation, where $\mathcal{A} = \{1, 2, \ldots, A\}$ represents $A$ different levels of privacy budgets. 
The action $a^{t}$ corresponds to the privacy budget $(\epsilon_u^{t}, \delta_u^t)$, denoted by $(\epsilon_u^{t}, \delta_u^t) = C_u(a^{t})$,  for generating DP noises on client $u$. 
Our objective is to maximize the final recommendation accuracy, gauged by reward $\sum_{t=1}^Tr^{t}$. Here, $r^{t}$ denotes the reward of round $t$, which can be defined as the change of training loss or recommender accuracy between round $t$ and $t-1$. The FL training performance is highly dependent on data quality. Thus,  $r^{t}$ should be related to the privacy budget allocated to round $t$ and training records contributed by clients. According to the process in Fig.~\ref{fig:system}, training records are represented by $\bm{\mathrm{X}}^{t}$. Thus, we define the objective as
$ \mathbb{E}\left[r^{t} \mid a^{t},\bm{\mathrm{X}}^{t}\right]$.
Given the limited total privacy budget, our problem can be formulated as follows:
\begin{subequations}
\begin{align}
\label{con:problem_reg}
\mathbb{P}1: \max_{\bm{a},\tau}&\; \sum_{t=1}^\tau \mathbb{E}\left[r^{t} \mid a^{t}, \bm{\mathrm{X}}^{t}\right], \\
 \text{s.t.} \quad &\sum_{t=1}^{\tau} (\epsilon_u^{t}, \delta_u^t) \leq (\epsilon_u^{total}, \delta_u^{total}), \ \forall u\in\mathcal{U}, \label{EQ:c1}\\
 &(\epsilon_u^{t}, \delta_u^t) = C_u(a^{t}),\ \forall u\in\mathcal{U},\forall t,\label{EQ:c2}\\
&a^{t} \in \mathcal{A}, \ \forall t,
 \ \tau \leq T \label{EQ:c3}. 
\end{align}
\end{subequations}
There are two challenges for solving $\mathbb{P}1$. First, we need to specify the expression of 
$\mathbb{E}\left[r^{t} \mid a^{t}, \bm{\mathrm{X}}^{t}\right]$. To our best knowledge, the exact expression is unavailable in existing works. Second,  it cannot be maximized by a greedy algorithm. If we only consider $\mathbb{E}\left[r^{t} \mid a^{t}, \bm{\mathrm{X}}^{t}\right]$ for round $t$, we should allocate all budget to $t$ such that  $\mathbb{E}\left[r^{t} \mid a^{t}, \bm{\mathrm{X}}^{t}\right]$ is maximized, which unfortunately depletes the privacy budget for other training rounds, resulting in poor final recommendation accuracy. 

\section{BGTplanner Design}\label{sec:algorithm design}

Our solution for  $\mathbb{P}1$ is depicted in Fig.~\ref{fig:BGTplanner}. Briefly speaking, Predictor employs the Gaussian process regression (GPR) to predict reward $\mathbb{E}\left[r^{t} \mid a^{t}, \bm{\mathrm{X}}^{t}\right]$ based on historical information. To avoid the trap of greedy budget allocation,  Constrainer properly penalizes each round allocation.  Allocator makes final budget allocation decisions by considering both the current improvement reward and penalized budget cost.  

\begin{figure}[t]      
\centering      
\includegraphics[width=\linewidth]{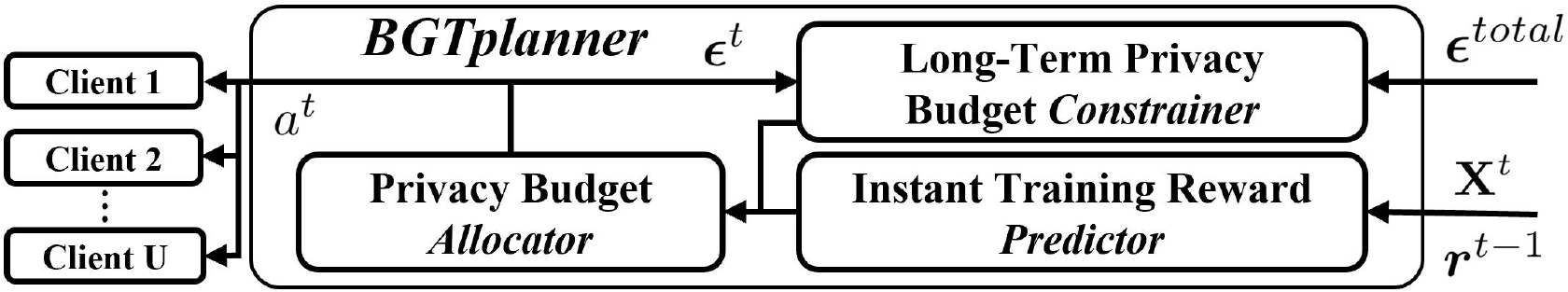}   
\caption{The workflow of BGTplanner.}
\label{fig:BGTplanner}  
\end{figure}

\subsection{Instant Training Reward Predictor Design}\label{sec:Reward Prediction}

In our problem,  accurate prediction of rewards serves as a cornerstone for BGTplanner. 
Given an action $a$, we employ GPR to model the relationship between the context input, \emph{i.e.}, $ \bm{\mathrm{X}}^{t}$, and the corresponding reward.  

In a typical recommender, $\bm{\mathrm{X}}^{t}$ is a 
high-dimensional matrix, making it impossible to directly apply GPR. To reduce the dimension of $ \bm{\mathrm{X}}^{t}$, we conduct the Singular Value Decomposition (SVD) operation to obtain $ \bm{\mathrm{X}}^{t} = \bm{\mathrm{U}}\Sigma \bm{\mathrm{V}}^*$. Then, we rank diagonal elements in $\Sigma$ by a descending 
order 
to only reserve top $d$ elements, where $d$ is a 
tuneable 
hyper-parameter. Let $\bm{x}^{t}$ denote the vector containing reserved $d$ elements. Unless otherwise specified, we use $\bm{x}^{t}$ to denote the context vector hereafter.

With a much lower dimension, GPR can predict $ \mathbb{E}\left[r^{t} \mid a, \bm{x}^{t}\right], \forall a\in\mathcal{A}$ based on historical information. 
To distinguish with $ \mathbb{E}\left[r^{t} \mid a, \bm{x}^{t}\right]$, let $\hat{r}(a, \bm{x}^{t})$  denote the predicted reward. 
According to~\cite{DBLP:conf/ijcai/YangZHWS23}, $\hat{r}(a, \bm{x}^{t})$ can be modelled as a random variable sampling value from the distribution of the Gaussian Process $\text{GP}\left(\mu_0(\bm{z}^{t}), k(\bm{z}^{t},\bm{z}')\right)$, where $\mu_0(\bm{z}^{t})$ denotes the mean function\footnote{For simplicity and practicability, $\mu_0(\bm{z})$ is assumed to be 
zero~\cite{DBLP:conf/ijcai/YangZHWS23}.} and $k(\bm{z}^{t},\bm{z}')$ is the covariance (or kernel) function for input vector $\bm{z}^{t}$ and the historical input vector $\bm{z}'$. 
To simplify our notation, we can concatenate any action $a \in \mathcal{A}$ (or a historical action $a^{t'}\in \mathcal{A}, \forall t' < t$) with the context $\bm{x}^{t}$ (or $\bm{x}^{t'}$) to form the vector $\bm{z}^{t} = [a, \bm{x}^{t}]$ (or $\bm{z}' = [a^{t'}, \bm{x}^{t'}]$). 
\footnote{
Note that $a^{t'}$ is already fixed in round $t'$ for $t'<t$.}

To adapt the GPR to model the temporal dynamics of observations (\emph{i.e.}, inputs $\bm{z}$), we construct a composite kernel that incorporates both the squared exponential kernel function and the Ornstein-Uhlenbeck temporal covariance function~\cite{DBLP:conf/icml/SrinivasKKS10}, yielding
\begin{equation}
k(\bm{z}^{t},\bm{z}')=(1-\alpha)^{\frac{|t-t'|}{2}} \exp  ( -\frac{||\bm{z}^{t}-\bm{z}'||_{2}}{2\,s^{2}} ), \forall \bm{z}'\in \mathcal{Z}^{t-1},
\end{equation} 
where $\mathcal{Z}^{t-1}=\{\bm{z}^{1},...,\bm{z}^{t-1}\}$ is the set of historical inputs until the round $t-1$, $|t - t'|$ represents the temporal difference between $\bm{z}^{t}$ and historical input $\bm{z}'$, and $\alpha$ is a hyper-parameter tuning the weight assigned to temporal correlations. 
The hyper-parameter $s > 0$ represents the length scale for determining the influence of one input on another. 
Hence, a larger $k(\bm{z}^{t}, \bm{z}')$ indicates a stronger correlation between $\bm{z}^{t}$ and $\bm{z}'$, implying the effectiveness of the historical observation with input $\bm{z}'$ in predicting the reward with input $\bm{z}^{t}$.  

Until  round $t$, we can collect the input set $\mathcal{Z}^{t-1}$ and the reward vector $\bm{r}^{t-1} =[r^{1},...,r^{t-1}]^{\top}$ including $t-1$ observed rewards. Thus, we can establish a GPR with $t-1$ variables to predict the $r^{t}$ distribution. The detailed joint distribution of GPR is 
presented 
in Appendix~B. 
Using the joint distribution of GPR and Bayes' theorem, we can predict $\hat{r}(a, \bm{x}^{t})$ for any action $a \in \mathcal{A}$ with context $\bm{x}^t$ by the conditional distribution $\mathcal{N}({\mu}(\bm{z}^{t}), \sigma(\bm{z}^{t}))$ parameterized by the mean function ${\mu}(\bm{z}^{t})$ in Eq.~\eqref{eq:cal_mean} and variance function $\sigma(\bm{z}^{t})$,\footnote{Due to the limited space, the expression of variance function $\sigma(\bm{z}^{t})$ 
is presented in Appendix~B.
} where $\bm{z}^{t} = [a, \bm{x}^{t}]$.
\begin{equation}
    \label{eq:cal_mean}
    {\mu}(\bm{z}^{t}) = {\mathbf{k}^{t-1}(\bm{z}^{t})}{[{\mathbf{K}}^{t-1}+\sigma_\mu^{2}\mathbf{I}]}^{-1}\bm{r}^{t-1}, \forall a\in\mathcal{A}.
\end{equation}
Here, $\mathbf{k}^{t-1}(\bm{z}^{t})=[k(\bm{z}^{t},\bm{z}^{1}),k(\bm{z}^{t},\bm{z}^{2}),\dots,k(\bm{z}^{t},\bm{z}^{t-1})]$ is the kernel vector between historical inputs and the current input, 
$\mathbf{K}^{t-1} = [k(\bm{z},\bm{z}')]_{\forall \bm{z},\bm{z}' \in \mathcal{Z}^{t-1}}$ represents the kernel matrix among historical inputs,
and $\mathbf{I}$ denotes an identity matrix.
Additionally, $\sigma_\mu$ is a variance parameter of distribution $\mathcal{N}_{\mu}(0, \sigma_\mu^{2})$, representing the deviation of the zero-mean random noise between the observed reward $r$ and the estimated reward $\hat{r}(a, \bm{x}^{t})$~\cite{DBLP:books/lib/RasmussenW06}. 

For a given $a$, it can be concatenated with the context $\bm{x}^{t}$ to form the input $\bm{z}^{t}$.
By leveraging Eq.~\eqref{eq:cal_mean}, we can estimate the expected reward $\hat{r}(a, \bm{x}^{t})$.
By enumerating $a\in\mathcal{A}$, we can predict the reward for choosing different actions. 
Note that our Predictor design is not based on a fixed $\bm{x}^t$. It can 
process dynamic contexts for predicting rewards with the evolving dynamics of user-item interactions over time. 





\subsection{Long-Term Privacy Budget Constrainer Design}\label{sec:Online Dual Variables Update}

We next discuss the design of the long-term privacy budget constrainer in BGTplanner. 

In $\mathbb{P}1$, the privacy budget is constrained in the entire training process by \eqref{EQ:c1}. To solve it, we decouple the time-dependent privacy budget constraint with an online queue optimization method~\cite{ouyang2023dynamic,DBLP:conf/aistats/HanZWXZ23} so that we can consider budget allocation round by round in $\mathbb{P}1$. 
For training round $\tau$, considering that the inequality $\sum_{t=1}^{\tau} \epsilon_u^{t} \leq \sum_{t=1}^{T} \epsilon_u^{t}$ always holds when $\tau \leq T$, we can relax the constraint in Eq.~\eqref{EQ:c1} as $\sum_{t=1}^{T} \epsilon_u^{t}\leq\epsilon_u^{total},\forall u$.
$\delta_u^t$ can be constrained in a similar way. For saving space, we simply use  $\epsilon_u^{t}$ to denote the privacy budget hereafter. 

The problem enforcing  long-term privacy budget constraints can be transformed into queue recursion optimization with Lagrange multipliers $\bm{\lambda}^{t}={ [\lambda_u^{t}]^U}^\top,\forall t.$ 
Therefore, the Lagrange function of reward in Eq.~\eqref{con:problem_reg} can be written as:
$L(\bm{a}, \bm{\lambda}) 
= \sum_{t=1}^T L^{t}(a^{t} ,\bm{\lambda}^{t}), 
$
where
\begin{equation}
\begin{split}
      &L^{t}(\,a^{t},\bm{\lambda}^{t})=
       - \mathbb{E}\left[r^{t} \mid a^{t},\bm{x}^{t}\right] +  \langle\frac{\bm{\epsilon}^{{total}}}{T} -\bm{\epsilon}^{t},\bm{\lambda}^{t}\rangle.
\end{split}
\end{equation}
Here, $\bm{\epsilon}^{total}=[\epsilon^{total}_u]^{U}$ is the total privacy budget vector,  $\bm{\epsilon}^{t}=[\epsilon^{t}_u]^{U}$  is the per-round privacy consumption vector for all clients  and $ \epsilon^{t}_u = C_u(a^t)$  is the function mapping actions and privacy budgets. 
Additionally, $\langle\cdot,\cdot\rangle$ represents the inner product of two vectors.
The problem $\mathbb{P}1$ is converted to maximizing the reward with the constraint punishment by solving 
$
\min_{\bm{a},\tau}
\max_{\bm{\lambda}}L(\bm{a}, \bm{\lambda})\; \text{s.t.}\; \eqref{EQ:c2},\eqref{EQ:c3}.
$ 
By penalizing the privacy budget consumption per training round, we can simplify the problem by only considering how to optimize $L^{t}(\,a^{t},\bm{\lambda}^{t})$ for a particular round $t$. 

By resorting to  the online mirror descent (OMD) algorithm~\cite{DBLP:journals/ftopt/Hazan16,DBLP:journals/ftml/Shalev-Shwartz12}, we can tackle the Lagrangian dual problem in every round $t$, formulated as $\min_{\bm{\lambda}^t\in\mathcal{V}} D^{t}(\bm{\lambda}^{t})$, where 
\begin{equation}
\label{con:dual_problem}
D^{t}(\bm{\lambda}^{t}) = \max_{a^{t}}   \left(\mathbb{E}\left[r^{t} \mid a^{t},\bm{x}^{t}\right] - \langle\frac{\bm{\epsilon}^{{total}}}{T} -\bm{\epsilon}^{t},\bm{\lambda}^{t}\rangle\right).
\end{equation}
Here, $\mathcal{V} = \left\{\boldsymbol{\lambda} \in \mathbb{R}^N : \boldsymbol{\lambda} \geq 0, \,\|\boldsymbol{\lambda}\|_1 \leq \Lambda\right\}$ is the space of the dual variable $\boldsymbol{\lambda}$, and $\Lambda > 0$ is the $l_1$-radius of $\mathcal{V}$. 

With the determined space set $\mathcal{V}$, Constrainer can update $\bm{\lambda}^{t+1}$ in the  rest training rounds with the following  rule:
\begin{equation}
    \label{eq:lambdaupdate}
     \bm{\lambda}^{t+1}=\arg \min _{\bm{\lambda} \in \mathcal{V}}\left(\langle\nabla D^{t}\left(\bm{\lambda}^{t}\right), \bm{\lambda}\rangle+\frac{B\left(\bm{\lambda}, \bm{\lambda}^{t}\right)}{\eta_t} \right), 
\end{equation}
where $B(\cdot,\cdot)$ is the Bregman divergence of the generating function and $\eta_t$ is the step factor of the OMD algorithm. Here, we select the generating function as the negative entropy function, the same as that in~\cite{DBLP:conf/aistats/HanZWXZ23}.

To sum up,  Constrainer enables the integration of long-term privacy budget constraints into instant reward maximization and focuses on dynamically tuning $\bm{\lambda}^t$ to penalize the privacy budget consumption properly. 

\subsection{Privacy Budget Allocator Design}\label{Pirvacy Budget Allocation}
The last component in BGTplanner is designed by leveraging contextual Multi-Armed Bandit (CMAB) to make final budget allocation decisions. 
Briefly speaking, there are two phases in BGTplanner, \emph{i.e.}, the \textit{initial} stage and the \textit{exploration-exploitation} stage, which are described as below.

\subsubsection{\textbf{Initial Stage}} The target of the \textit{initial} stage is to  establish the space $\mathcal{V}$ of $\bm{\lambda}^t$ for Constrainer, so that 
the privacy budget consumption can be properly penalized  in the \textit{exploration-exploitation} stage.
According to~\cite{DBLP:conf/aistats/HanZWXZ23}, we employ a regression-based algorithm to determine $\Lambda$. 
The \textit{initial} stage lasts $(A + 1) \cdot T_0$ training rounds to evaluate $\Lambda$. Here, $A$ is the size of the action space, and $T_0$ is a hyper-parameter representing the number of times each action is selected.

To ensure that the OMD algorithm achieves $\text{O}(\sqrt{T})$ regret bounds, our preferred choice is to set $\Lambda = \frac{T \cdot \text{OPT}}{\epsilon}$, where OPT is defined in Eq.~\eqref{eq:static_policy}~\cite{DBLP:conf/aistats/HanZWXZ23}. However, OPT cannot be solved accurately in an online scenario.
Therefore, we employ a regression algorithm in the \textit{initial stage} including $(A + 1)\cdot T_0$ training rounds to evaluate $\Lambda$.
Generally speaking, there are two steps in the \textit{initial stage}. 

\textbf{Step \textcircled{1}:} For the first $A\cdot T_0$ rounds, we play different $A$ actions in CMAB evenly for $T_0$ times to gather historical records set $\mathcal{H}^{A\cdot T_0}=\{(\bm{z}^1, r^1),\cdots, (\bm{z}^{A\cdot T_0}, r^{A\cdot T_0})\}$ for all inputs $\bm{z}^{t'}=[a^{t'},\bm{x}^{t'}]$ and reward outputs $r^{t'}$, where ${t'} \in \{1,\cdots, A\cdot T_0\}$.
With the historical records, we can formulate the GPR model as an oracle to predict the reward  $\hat{r}(a,\bm{x}^t)$ and the mean function of reward $\mu(\bm{z}^t)$ with input $\bm{z}^t = [a^t,\bm{x}^t]$, where $t\in\mathcal{T}_0$ and $\mathcal{T}_0$ is defined in Eq.~\eqref{eq: T_0}. 
\begin{subequations}\label{eq:linear_programming} \begin{align}  
&\hat{\text{OPT}}(T_0) = \max_{\bm{o} \in\mathcal{O}} \frac{1}{T_0} \sum_{t \in \mathcal{T}_0} \sum_{a \in\mathcal{A}}  \hat{r}(a,\bm{x}^t) \cdot o_{a,t}, \label{eq: linear_programming objective}\\ 
\text{s.t.}~\frac{1}{T_0} &\sum_{t \in \mathcal{T}_0} \sum_{a\in\mathcal{A}} C_u(a) \cdot o_{a,t}  \leq \frac{\epsilon_{min}^{total}}{T}+2 M\left(T_0\right), \\
\mathcal{T}_0  &=\left\{t: A\cdot T_0+1 \leq t \leq (A+1) \cdot T_0\right\}, \label{eq: T_0}\\  
M\left(T_0\right)  &=\sqrt{A\cdot {E}\left(T_0 \right)+4\cdot \frac{\log (T\cdot U)}{T_0}}, \label{EQ:M(T_0)}\\ 
E\left(T_0 \right)  &= \frac{1}{T_0}\sum_{t\in \mathcal{T}_{0}}(\mu(\bm{z}^{t})-r^{t})^2,\bm{z}^t=[a^t,\bm{x}^{t}],\label{eq: error function of linear_programming}\\
\epsilon_{min}^{total} &= \underset{u}{\arg\min}\{\epsilon_u^{total} \mid \forall u \in \mathcal{U}\},
\end{align}
\end{subequations} 

\textbf{Step \textcircled{2}:} For the latter total $T_0$ rounds, we select actions $a^t$ randomly and observe the reward $r^t$, where $t\in\mathcal{T}_0$. 
Then, we use the observed rewards and GPR model to estimate $\hat{\text{OPT}}(T_0)$ by solving the linear programming problem~\eqref{eq:linear_programming}, where $\bm{o}$ is the solution matrix and we use the GPR model as the oracle to predict the reward $\hat{r}(a,\bm{x}^t)$ and the mean of reward $\mu(\bm{z}^t)$. 
Additionally, $T$ is the maximum training rounds, $U$ is the total number of clients, $T_0$ is the hyper-parameter of rounds to select each action, and $\epsilon_{min}$ is the smallest privacy budget among all clients.
Therefore, we can obtain all the parameters needed to solve problem~\eqref{eq:linear_programming} using any general linear programming solver.

Upon obtaining the estimated value of $\hat{\text{OPT}}\left(T_0\right)$, we can set
\begin{equation}\label{eq: estimated OPT}
    \Lambda=\frac{T}{\epsilon_{min}^{total}}\left(\hat{\text{OPT}}\left(T_0\right)+M\left(T_0\right)\right),
\end{equation}
where $M(T_0)$ is defined in Eq.~\eqref{EQ:M(T_0)}.
Consequently, the error between $\hat{\text{OPT}}\left(T_0\right)$ and OPT is bounded by the estimated error guarantee, as demonstrated in Lemma~\ref{OPT_Z} in Appendix~E, based on the analysis in \cite{DBLP:conf/aistats/HanZWXZ23}.
The detailed algorithm for determining
$\Lambda$ is provided in Alg.~\ref{alg: Initial Stage}.

\begin{algorithm}[!t]
\caption{Initial Stage to Estimate $\Lambda$ of the Server.} 
\label{alg: Initial Stage} 
\textbf{Input:} Training rounds $T$; Exploration rounds $T_0$; Total privacy budget $\bm{\epsilon}^{total}$, Item embedding $\bm{\iota}^{0}$, user embedding $\bm{\nu}^{0}$, Recommender parameters $\bm{\omega}^{0}$. \\
\textbf{Output:} Radius $\Lambda$; Remained budget $\bm{\epsilon}^{re}$; Historical record set $\mathcal{H}^{(A+1)\cdot T_0}$. 
\begin{algorithmic}[1] 
\STATE Distributing $\bm{\iota}^{0}$, $\bm{\nu}^{0}$, $\bm{\omega}^{0}$ to all clients as  the initialization of their local model;\\
\textbf{// Step \textcircled{1} //}\\
\STATE Initializing remained budget $\bm{\epsilon}^{re}\leftarrow\bm{\epsilon}^{total}$ and historical records $\mathcal{H}^0=\emptyset$ for GPR model;
\FOR{ $a^t\in\mathcal{A}$} 
\FOR{$t\in\{(a^t-1)\cdot T_0 + 1,\dots,a^t \cdot T_0\}$} 
 \STATE Receiving the $\bm{\mathrm{X}}^{t}$ from all clients and generate the context $\bm{x}^{t}$ to concatenate $\bm{z}^{t} = [a^t, \bm{x}^{t}]$;
 \STATE 
 Sending $a^{t}$ to all clients and executing \textbf{DPFR TRAINING} (Alg.~\ref{alg:DPFR TRAINING}) to obtain $r^{t}, \bm{\epsilon}_u^{t}$; 
 \STATE Updating $\bm{\epsilon}_u^{re}\leftarrow \bm{\epsilon}^{re}_u -  \bm{\epsilon}_u^{t}$;\\
 \STATE Collecting $\mathcal{H}^t\leftarrow\mathcal{H}^{t-1} \cup (\bm{z}^{t},r^{t}$); 
\ENDFOR 
\ENDFOR 
\STATE Modeling GPR with $\mathcal{H}^{A\cdot T_0}$;\\
\textbf{// Step \textcircled{2} //}\\
\STATE Generating $\mathcal{T}_0$ by Eq.~\eqref{eq: T_0} with $A$ and $T_0$; 
\FOR{$t\in \mathcal{T}_0$} 
\STATE Receiving the $\bm{\mathrm{X}}^{t}$ from all clients and generating the context $\bm{x}^{t}$;
\STATE Sending $a^{t}$ to all clients and executing \textbf{DPFR TRAINING} (Alg.~\ref{alg:DPFR TRAINING}) to obtain $r^{t}, \bm{\epsilon}_u^{t}$; 
 \STATE Updating $\bm{\epsilon}_u^{re}\leftarrow \bm{\epsilon}^{re}_u -  \bm{\epsilon}_u^{t}$;
  \STATE Concatenating $\bm{z}^{t} = [a^t, \bm{x}^{t}]$ and collecting $\mathcal{H}^t\leftarrow\mathcal{H}^{t-1} \cup (\bm{z}^{t},r^{t}$); 
  \STATE Updating GPR model with $\mathcal{H}^{t}$;
\ENDFOR 
\STATE Enumerating $a\in\mathcal{A}$ and $ t\in \mathcal{T}_0$ to calculate $\hat{r}(\cdot,\cdot)$ with GPR model and $\mathcal{H}^{(A+1)\cdot T_0}$;
\STATE Calculating $\mu(\cdot)$ with $\bm{z}^{t},\forall t \in \mathcal{T}_0$ by Eq~\eqref{eq:cal_mean};
\STATE Estimating $\hat{\text{OPT}}\left(T_0\right)$ by solving the linear programming \eqref{eq:linear_programming} with $\hat{r}(\cdot)$ and $\mu(\cdot)$ and $\mathcal{H}^{(A+1)\cdot T_0}$;
\STATE Determining the radius $\Lambda$ with $\hat{\text{OPT}}\left(T_0\right)$ by Eq.~\eqref{eq: estimated OPT}.
\end{algorithmic}
\end{algorithm}

\begin{algorithm}[!t]
\caption{DPFR TRAINING (One Round)} 
\label{alg:DPFR TRAINING} 
\textbf{Input:} Training round $t$,
 Clients set $\mathcal{U}$, Budget allocation  $a^{t}$.\\
\textbf{Output:} Training reward $r^{t}$, Budget consumption $\bm{\epsilon}^{t}$.\\
\textbf{// Clients $u\in\mathcal{U}$ //}
\begin{algorithmic}[1]
\STATE Training the local item embedding $\bm{\iota}^{t}_u$, user embedding $\bm{\iota}^{t}_u$, and recommender model $\bm{\omega}^{t}_u$ with local private data;
\STATE Generating DP noise $\bm{n}_u = [\bm{n}_\omega,\bm{n}_\iota]$ based on budget allocation $a^t$ and specific DP mechanism; 
\STATE Adding DP noise to the local gradient $\nabla\tilde{\bm{\omega}}^t_u = \nabla \bm{\omega}^{t}_u+\bm{n}_\omega$, $\nabla \tilde{\bm{\iota}}_u^{t} = \nabla \bm{\iota}_u^{t} + \bm{n}_\iota$;
\STATE Uploading the noisy gradient $\nabla\tilde{\bm{\omega}}^t_u$ and $\nabla \tilde{\bm{\iota}}_u^{t}$ to server;
\end{algorithmic}
\textbf{// Server //}
\begin{algorithmic}[1]
\STATE Receiving the noisy gradient of item embedding $\nabla\tilde{\bm{\iota}}^{t}_u$ and recommender model $\nabla\tilde{\bm{\omega}}^t_u$ of from all clients;
\STATE Performing model aggregation to obtain global $\bm{\iota}^{t}$, $\bm{\omega}^{t}$ with  $\nabla\tilde{\bm{\iota}}^{t}_u,\ \nabla\tilde{\bm{\omega}}^t_u,\ \forall u \in \mathcal{U}$ and training information $\mathbf{X}^t$;
\STATE Distributing the aggregated parameters $\bm{\iota}^{t}$ and $\bm{\omega}^{t}$ to all clients and observing the training reward $r^{t}$; 
\STATE Calculating the consumption budgets $\bm{\epsilon}^{t}$ for all clients.
\end{algorithmic}
\end{algorithm}

\begin{algorithm}[t]
  \caption{The BGTplanner Algorithm of Server.} 
 \label{alg:The BGTplanner Algorithm of Server} 
 \textbf{Input:} Training rounds $T$; Exploration rounds $T_0$; Total privacy budget $\bm{\epsilon}^{total}$. 
\begin{algorithmic}[1]  
\STATE Determining the radius $\Lambda$ by Alg.~\ref{alg: Initial Stage} and obtaining $\bm{\epsilon}^{re}$ and $\mathcal{H}^{(A+1)\cdot T_0}$;\label{line: obtain Z}
\STATE Randomly initializing the dual variables $\bm{\lambda}^{(A+1)\cdot T_0 + 1}$ by Eq.~\eqref{eq:lambdaupdate} in the radius $\Lambda$; 
\FOR{$t\in\{(A+1)\cdot T_0 + 1,\cdots,T\}$}
\STATE \algorithmicif\ All clients exhaust their privacy budget 
\algorithmicthen\ Exiting model training \algorithmicend\ \algorithmicif; \\
\textbf{// Step~\textcircled{1} //}
\STATE Receiving the $\bm{\mathrm{X}}^{t}$ from all clients and generating the context $\bm{x}^{t}$;\label{line: server collect}
\STATE Predicting $\mu(\bm{z}^t)$ with $\bm{z}^t = [a, \bm{x}^{t}], a\in\mathcal{A}$ by Eq.~\eqref{eq:cal_mean};
\STATE Obtaining $\beta^t(a), \forall a \in \mathcal{A}$ with $\mu(\bm{z}^t)$ and $\bm{\lambda}^{t}$ based on Eq.~\eqref{eq:rewardfunction};\label{line:score function}\\
\textbf{// Step~\textcircled{2} //}
\STATE Calculating $p^{t}(a), \forall a \in \mathcal{A}$ with $\beta^{t}(a)$ based on Eqs.\eqref{eq: probability_action_a}-\eqref{eq: probability_action_b}; \label{line: sample action begin}
\STATE Sampling a budget allocation action $a^{t} \sim \bm{p}^{t}$; \label{line: sample action end}
\STATE Sending $a^{t}$ to all clients and executing \textbf{DPFR TRAINING} (Alg.~\ref{alg:DPFR TRAINING}) to obtain $r^{t}, \bm{\epsilon}_u^{t}$; \\
\textbf{// Step~\textcircled{3} //}
\STATE Updating $\bm{\epsilon}_u^{re}\leftarrow \bm{\epsilon}^{re}_u -  \bm{\epsilon}_u^{t}$;
\STATE Concatenating $\bm{z}^{t} = [a^t, \bm{x}^{t}]$;
\STATE Collecting $\mathcal{H}^t\leftarrow\mathcal{H}^{t-1} \cup (\bm{z}^{t},r^{t}$); 
\STATE Updating GPR model with $\mathcal{H}^{t}$;
\STATE Updating the dual variables $\bm{\lambda}^{t+1}$ by Eq.~\eqref{eq:lambdaupdate} with $\bm{\epsilon}_u^{t}$ in the radius $\Lambda$. \label{line: server update end}
\ENDFOR
\end{algorithmic}
\end{algorithm}


\subsubsection{\textbf{Exploration-Exploitation Stage}} In this stage, BGTplanner utilizes the outputs of Predictor and Constrainer, and the CMAB framework to allocate the privacy budget. 
In general, the \textit{Exploration-Exploitation} stage includes three steps in each round $t\in\left((A + 1) \cdot T_0,T\right]$, described as follows:

\textbf{Step~\textcircled{1} (Generating Action Scores)}: After observing context $\bm{x}^t$, BGTplanner acquires the estimated mean of rewards ${\mu}(\bm{z}^{t})$ by Predictor, where $\bm{z}^{t} = [a,\bm{x}^t],\forall a\in\mathcal{A}$, and calculates the predicted Lagrangian as an appropriate score function by  considering both rewards and penalties:
\begin{equation}
\label{eq:rewardfunction}
\beta^t(a)=\mu(\bm{z}^{t})-\langle\frac{\bm{\epsilon}^{total}}{T} -\bm{\epsilon}^t,\bm{\lambda}^{t}\rangle,\forall a \in\mathcal{A},
\end{equation}
where $\bm{\lambda}^t$ is updated by Constrainer in round $t-1$.

\textbf{Step~\textcircled{2} (Selecting the Optimal Action)}: Based on the estimated score for taking action $a$, BGTplanner generates probabilities for all arms corresponding to different privacy budget allocation actions in the CMAB model. Let $a^{max}$ denote the action with the highest score $\beta^{t}({a}^{max})$, the probability to sample each action can be calculated as:
\begin{equation}
\label{eq: probability_action_a}
    p^{t}(a)=\frac{1}{A+\gamma\,\left(\,\beta^{t}({a}^{max})-\beta^{t}({a})\,\right)}, \forall a\in\mathcal{A}-\{{a}^{max}\}, 
\end{equation}
where $\gamma$ is a hyper-parameter representing the extent to which BGTplanner prefers exploitation over exploration. If $\gamma $ is smaller, $ p^{t}(a)$ is closer to $\frac{1}{A}$, implying that BGTplanner samples actions in a more uniform manner, and hence prefers exploration.  The probability of ${a}^{max}$ is: 
\begin{equation}
\label{eq: probability_action_b}
    p^{t}({a}^{max})=1-\sum_{a\in\mathcal{A}, a \neq {a}^{max}} p^{t}(a).
\end{equation}
The final action $a^{t}$ is sampled with the probability distribution in Eqs.\eqref{eq: probability_action_a} and \eqref{eq: probability_action_b} under context $\bm{x}^{t}$ and  the variable $\bm{\lambda}^{t}$.

\textbf{Step~\textcircled{3} (Updating Penalties)}:
BGTplanner leverages Constrainer to update the new dual variable vector $\bm{\lambda}^{t+1}$, which will be used for budget allocation in the next training round on the fly.


\subsection{Details of BGTplanner}\label{sec: BGTplanner algorithm}
We provide a detailed description of the BGTplanner algorithm in Alg.~\ref{alg:The BGTplanner Algorithm of Server}. The algorithm operates in two stages: an initial stage, described in Alg.~\ref{alg: Initial Stage}, and a subsequent exploration-exploitation stage, which includes three main steps in each round. In each training round of the second stage, BGTplanner predicts the scores for all actions, selects the optimal action and adjusts the dual variables based on observed rewards and the remaining privacy budgets.
With the chosen budget allocation action, the server and all clients execute one round of DPFR training, as detailed in Alg.~\ref{alg:DPFR TRAINING}, to update the model and embedding parameters. This process navigates a dynamic training environment, where the privacy budget is strategically allocated to maximize model performance while adhering to privacy constraints.
The iterative process continues until the privacy budget is exhausted or the desired training performance is achieved.

\section{Analysis and Implementation}
Additionally, we undertake privacy, regret and complexity analysis to substantiate our approach's theoretical guarantee.

\subsection{Privacy Analysis.}
Given $(\epsilon^{total}_u,  \delta^{total}_u)$, each client can allocate its budget as follows. If the Laplace mechanism is adopted,  $\delta^{t}_u = 0$ and each client can employ the native composition rule~\cite{Kang2020} to accumulate the total privacy loss $\epsilon^{total}_u$ during the multiple rounds.  Otherwise, if the Gaussian mechanism is adopted, clients can convert the budget to the RDP (R\'enyi Differential Privacy) budget, which can then be simply split into multiple training rounds~\cite{Girgis2021}. In this case, Eq.~\eqref{EQ:c1} should be updated by the RDP budget. 
It has been proved in \cite{Kang2020, Girgis2021} that if clients follow these methods to accumulate their privacy loss,  $(\epsilon^{total}_u,  \delta^{total}_u)$-DP is satisfied on client $u$. 

\subsection{Overhead Analysis.} BGTplanner incurs additional computation and memory overhead for making budget allocation decisions. Through analysis, we show that the overhead cost is lightweight and affordable for the server.  For Predictor, the complexity for calculating the Gaussian Kernel is $\mathrm{O}(A\cdot d\cdot t)$, with $d$ dimension of the input vector. The computational complexity for estimating the reward $\mu(\bm{z}^{t})$ in Eq.~\eqref{eq:cal_mean} is $\mathrm{O}(A\cdot d\cdot t^2)$.
Calculating the score vector $\beta^{t}(a)$ by Eq.~\eqref{eq:rewardfunction} incurs a complexity of $\mathrm{O}(A \cdot U)$, and computing the probability vector $\bm{p}^{(t)}$ via Eq.~\eqref{eq: probability_action_a}-Eq.~\eqref{eq: probability_action_b} has a  complexity of $\mathrm{O}(A)$. 
To conclude, the total computation complexity per round 
is $\mathrm{O}(A\cdot(t^2\cdot d+U))$. 
According to \cite{DBLP:journals/ftopt/Hazan16}, the computational complexity of Constrainer is also small  by using the MOD algorithm.

Additionally, the memory complexity is $\mathrm{O}(A \cdot t^2 + U)$, primarily because of the storage requirements for 1) the GP kernels, 2) estimated means, scores and probabilities for all actions, and 3) client budgets and penalties information.

\subsection{Regret Analysis.}
We then discuss the regret performance of BGTplanner. 
We define the optimal expected reward for one round while ensuring resource utilization remains within a predefined budget as $\text{OPT}$~\cite{DBLP:conf/aistats/HanZWXZ23}, \emph{i.e.},
\begin{definition}\label{def:OPT}
Let $\hat{r}^*(a, \bm{x})$ denote the optimal regression to characterize the expectation of reward distribution and the static randomized policy $p_a^*:\mathcal{X}\rightarrow P_\mathcal{A}$ as the optimal mapping function (defined on context distribution space $\mathcal{X}$) assigning probabilities to actions in $\mathcal{A}$. We have
\begin{equation} 
\begin{aligned}  
 &\text{OPT} = \max_{p_a^*:\mathcal{X}\rightarrow P_\mathcal{A}}\ \mathbb{E}_{\bm{x} \sim P_{\mathcal{X}}}\left[\sum_{a \in\mathcal{A}} p_a^*( \bm{x})\, \hat{r}^*(a, \bm{x})\right], \\ 
 &\text { s.t. }\mathbb{E}_{\bm{x} \sim P_{\mathcal{X}}}\left[\sum_{a \in\mathcal{A}} p_a^*( \bm{x}) \, C_u(a)\right] \leq \frac{\epsilon_u^{total}}{T}\,,\, \forall u\in\mathcal{U}.
\end{aligned} 
\label{eq:static_policy} 
\end{equation}
Here, $P_\mathcal{A}$ and $P_{\mathcal{X}}$ represent  probability distribution over the action set $\mathcal{A}$ and context space $\mathcal{X}$, respectively.  $C_u: \mathcal{A} \rightarrow \mathbb{R}_+$ is a mapping from action to privacy budget consumption.
\end{definition}

With Definition~\ref{def:OPT}, let $T\cdot \text{OPT}$ denote the upper bound of the optimal reward in online scenarios~\cite{DBLP:conf/nips/AgrawalD16}, we can minimize the regret between the OPT and the actual reward obtained by BGTplanner online algorithm as:
\begin{equation}  
 \text{Reg}(T)=T \cdot \text{OPT} - \sum_{t=1}^T \mathbb{E}\left[r^{t} \mid a^{t},\bm{x}^{t}\right].
\end{equation}

We first give the error guarantee of the estimation of $\ell_1$-radius $\Lambda$.
\begin{lemma}   
\label{OPT_Z}
Denoting the optimal value of \eqref{eq:linear_programming} by $\hat{\text{OPT}}\left(T_0\right)$, and setting $\Lambda=\frac{T}{\epsilon_{min}^{total}}\left(\hat{\text{OPT}}\left(T_0\right)+M\left(T_0\right)\right)$, we have with probability at least $1-O\left(1 / T^2\right)$,
\begin{equation}
    \frac{T \cdot \text{OPT}}{\epsilon^{total}_{min}} \leq \Lambda \leq\left(\frac{6 T \cdot M\left(T_0\right)}{\epsilon^{total}_{min}}+1\right)\left(\frac{T \cdot \text{OPT}}{\epsilon^{total}_{min}}+1\right),
\end{equation}  
\end{lemma}
Here, when $\epsilon^{total}_{min}=\mathrm{O}(T\cdot M(T_0))$, we can donate $\Lambda\lesssim \frac{T \cdot \text{OPT}}{\epsilon^{total}_{min}}+1$, where the notation `$\lesssim$' denotes a represents multiple that is less than. The comprehensive proof of Lemma~\ref{OPT_Z} can be found in Lemma 4.1 in~\cite{DBLP:conf/aistats/HanZWXZ23} .
With 
the estimation error guarantee of $\Lambda$ in Lemma~\ref{OPT_Z}, the OMD to update $\bm{\lambda}$ 
achieves optimal regret bounds $\mathrm{O}(\sqrt{T})$~(Lemma 2.2 in~\cite{DBLP:conf/aistats/HanZWXZ23}).
Therefore, we can obtain the following regret guarantee of Alg.~\ref{alg:The BGTplanner Algorithm of Server} and we rewrite the theorem as follows:
\begin{theorem} \label{theorem2}
    Denote $T$ as the total training rounds, $T_0$ as the number of times  each action is selected in the initial phase, and $A$ as the number of actions, we have $\text{Reg}(T)=T \cdot \text{OPT} - \sum_{t=1}^T \mathbb{E}\left[r^{t} \mid a^{t},\bm{x}^{t}\right]$ of BGTplanner satisfying:
    \begin{equation}
     \begin{aligned}   \text{Reg}(T) \lesssim & \left(\frac{T \cdot \text{OPT}}{\epsilon_{min}^{total}}+1\right) \sqrt{A\cdot T\left[\operatorname{Reg}^r(T)+1\right]} \\     & +\left(\frac{T \cdot \text{OPT}}{\epsilon_{min}^{total}}+1\right) A\cdot T_0,
    \end{aligned}   
    \end{equation}
where  $\epsilon_{min}^{total} = \arg\min_{u}\{\epsilon_u^{total} \mid \forall u \in \mathcal{U}\}$ is the minimum budget among the clients. Meanwhile, 
$\epsilon_{min}^{total}$  constrains $T_0$ with $\epsilon_{min}^{total}>\max \left\{(A+2)\cdot T_0, T \cdot M\left(T_0\right)\right\}$, where $M(\cdot)$ is the bias function of constraints of a linear programming problem defined in Appendix~C. 
$\operatorname{Reg}^r(T)$ is the regret between the optimal regression function $R^*(a, \bm{x})$ and GPR-based reward Predictor within total $T$ rounds, with an upper bound $\mathrm{O}(\sqrt{T\cdot \log(T)})$ proved in~\cite{DBLP:conf/icml/SrinivasKKS10}.
\end{theorem}

\begin{proof}
We can directly deduce the regret from the Theorem 4.1 in~\cite{DBLP:conf/aistats/HanZWXZ23} by considering reward regression oracle $\hat{r}(a, \bm{x})$ as GPR model with an upper bound $\mathrm{O}(\sqrt{T\cdot \log(T)})$ regret~\cite{DBLP:conf/icml/SrinivasKKS10} and constraint regression oracle $C(a)$ as an unbiased estimation.
In Alg.~\ref{alg:The BGTplanner Algorithm of Server}, we use $\text{O}(A\cdot T_0)$ rounds in the initial stage, the expected regret incurred by Alg.~\ref{alg: Initial Stage} is upper bounded by $\text{O}(\frac{T\cdot \text{OPT}}{\epsilon_{min}^{total}}+1)\cdot A \cdot T_0$, with the guarantee by $\Lambda \lesssim \frac{T \cdot \text{OPT}}{\epsilon_{min}^{total}}+1$ in Lemma~\ref{OPT_Z}. For the second stage in Alg.~\ref{alg:The BGTplanner Algorithm of Server}, the expected regret for the exploration-exploitation stage Alg.~\ref{alg:The BGTplanner Algorithm of Server} is upper bounded by $\left(\frac{T \cdot \text{OPT}}{\epsilon_{min}^{total}}+1\right) \sqrt{A \cdot T\left[\operatorname{Reg}^r(T)+1\right]}$, where $\sqrt{A \cdot T\left[\operatorname{Reg}^r(T)+1\right]}$ is the upper bound of CMAB method employing GPR model as reward regression oracle and OMD method to update dual variables online in total $T$ rounds (Theorem 4.1~\cite{DBLP:conf/aistats/HanZWXZ23}). 
Then, Theorem \ref{theorem2} holds by adding the regret in these two stages together.
\end{proof}

Theorem \ref{theorem2} guarantees that BGTplanner achieves an $\mathrm{O}(\sqrt{T})$ regret bound. Such a sub-linear regret bound is highly desirable, indicating that BGTplanner will not significantly diverge from the optimal budget allocation solution even under dynamic contexts. 

\begin{figure*}[t] 
\centering
    \subfigure{  \includegraphics[width=0.99\linewidth]{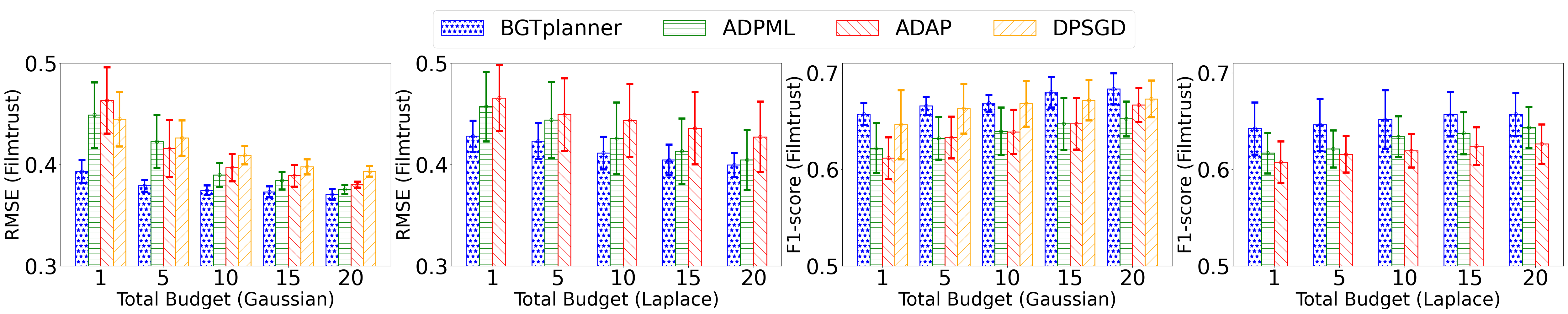}
    \vspace{-2mm}
    }
    \subfigure{     \includegraphics[width=0.99\linewidth]{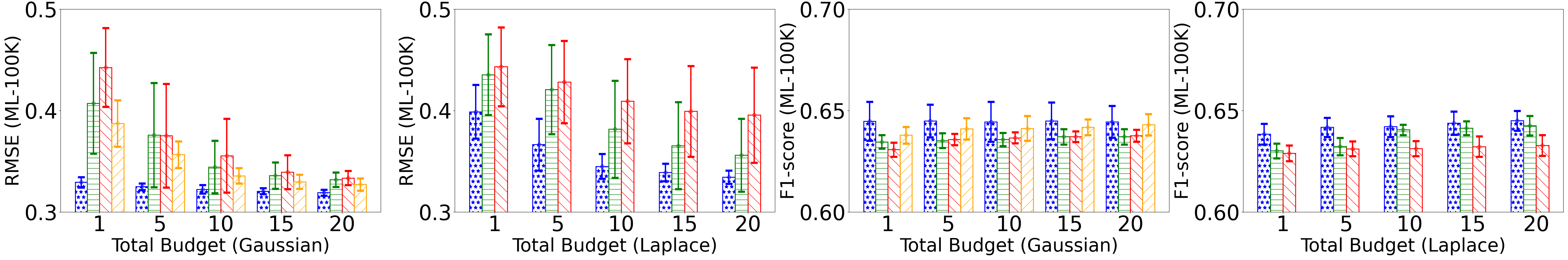}}
    \vspace{-2mm}
    \label{fig:all_baselines_ML-100K}
    \subfigure{     \includegraphics[width=0.99\linewidth]{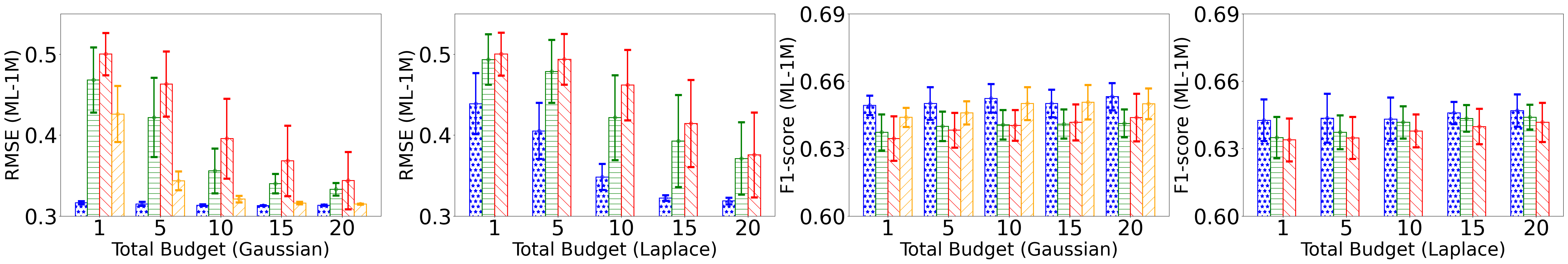}
\label{fig:all_baselines_ML-1M_all}
    }
\caption{Performance of all baselines across different settings of privacy budgets and DP noise mechanisms in Filmtrust, ML-100K, and ML-1M dataset, respectively. Lower RMSE values indicate better performance, while higher F1-scores reflect better performance.}
\label{fig:all_baselines}
\end{figure*}

\begin{table*}[t]
\centering
\renewcommand{\arraystretch}{1.25}
\caption{Performance in RMSE and F1-score under various datasets and noise models with the default privacy budget $\epsilon_u^{total}=10$. All methods are tested five times, and the average of the results is presented.}
\label{table:performance}
\resizebox{1\linewidth}{!}{
\begin{tabular}{|c|cc|cc|cc|cc|cc|cc|}
\toprule
\multirow{3}{*}{Methods}& \multicolumn{6}{c|}{RMSE\, ($\downarrow$) }  & \multicolumn{6}{c|}{F1-score\, ($\uparrow$)}         \\ \cline{2-13}
& \multicolumn{2}{c|}{Filmtrust}             & \multicolumn{2}{c|}{ML-100K}   & \multicolumn{2}{c|}{ML-1M}               & \multicolumn{2}{c|}{Filmtrust}             & \multicolumn{2}{c|}{ML-100K}   & \multicolumn{2}{c|}{ML-1M}               \\ 
&GM  &   LM & GM &  LM & GM & LM & GM  &   LM& GM & LM  & GM & LM\\ 
\midrule
FedSGD  & \multicolumn{2}{c|}{0.2905} & \multicolumn{2}{c|}{0.2017} & \multicolumn{2}{c|}{0.2006}  & \multicolumn{2}{c|}{0.8354}    & \multicolumn{2}{c|}{0.9018}   & \multicolumn{2}{c|}{0.9137} \\
OURS & \textbf{0.3746}& \textbf{0.4114}  & \textbf{0.3223} &\textbf{0.3449} & \textbf{0.3135} &      \textbf{0.3484}          & \textbf{0.6685}  & \textbf{0.6517} & \textbf{0.6444}&  \textbf{0.6421}& \textbf{0.6524}&      \textbf{0.6431}\\  
ADPML & 0.3898& 0.4258  &  0.3442  & 0.3815  & 0.3558 & 0.4217 & 0.6395&  0.6338 &  0.6358 &  0.6404  & 0.6407 &  0.6417\\ 
ADAP & 0.3970  & 0.4435 &  0.3553  & 0.4092 & 0.3958 & 0.4621 &  0.6389 &  0.6194  &  0.6366& 0.6312& 0.6404& 0.6379 \\
DPSGD   &0.4093  & \diagbox[]{}{}&0.3356 & \diagbox[]{}{} & 0.3210 &\diagbox[]{}{}& 0.6680 & \diagbox[]{}{}&0.6412 &  \diagbox[]{}{} & 0.6501 &\diagbox[]{}{}\\ 
\bottomrule
\end{tabular}
}
\small
\begin{tabbing}
\textbf{Note:} Lower is better ($\downarrow$). Higher is better ($\uparrow$). 
\end{tabbing}
\vspace{-4mm}
\end{table*}

\section{Performance Evaluation}\label{sec:perfomance}
In this section, we report the experimental results for comparing BGTplanner and the state-of-the-art baselines. To facilitate the peer review, we also anonymously open the source code of our system deployment\footnote{https://anonymous.4open.science/r/BGTplanner-AAAI}.


\subsection{Experimental Settings}
\subsubsection{Datasets} For our experiments, we exploit public rating datasets Filmtrust~\cite{Monti2017}, MovieLens 100k (ML-100K)~\cite{Harper2015}, and MovieLens 1M (ML-1M)~\cite{Harper2015}, which are commonly used in evaluating recommender systems. 
Filmtrust includes 18,662 ratings from 874 users, ML-100k includes 100,000 ratings for 1,682 movies from 943 users, and ML-1M includes 1,000,209 ratings for 3,952 movies from 6,040 users. All rating scores are normalized to the range $[0,1]$.
We implement FedGNN~\cite{wu2021fedgnn} as the recommendation model and set the number of clients for FL according to the number of users in Filmtrust and ML-100K, respectively. Due to the larger number of users in ML-1M, we randomly allocate its 6,040 users to 605 clients. 

Additionally, we split the training set and test set by the ratio of 4:1 for all rating datasets.
To further simulate the dynamic training contexts with continuously generated rating records, the training set is further randomly divided into two equal-size parts. At the beginning of training, only ratings in the first part can be used for training. As the training progresses, training samples in the second part are  randomly and evenly sampled and added to the first part for enhance training. 
Besides, all clients randomly select pseudo items from the set of non-interacted items for training FedGNN in each round~\cite{wu2021fedgnn}, and the number is set to 50, 50, and 400 for Filmtrust, ML-100K, and ML-1M, respectively. 

\begin{figure*}[t] 
\centering
    \subfigure{  \includegraphics[width=0.985\linewidth]{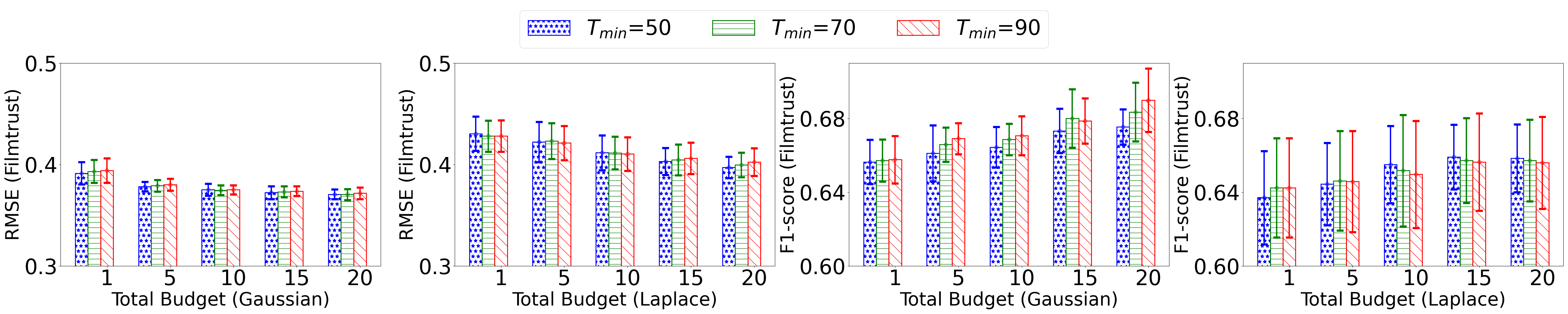}
    \label{fig:T_min_Filmtrust_all}
    }
    \subfigure{     
\includegraphics[width=0.985\linewidth]{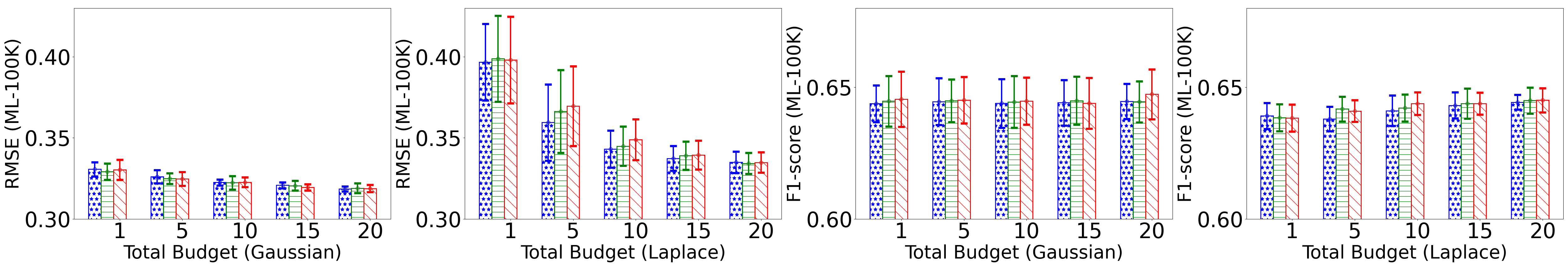}
    }
    \label{fig:T_min_ML-100K_all}
    \subfigure{     \includegraphics[width=0.985\linewidth]{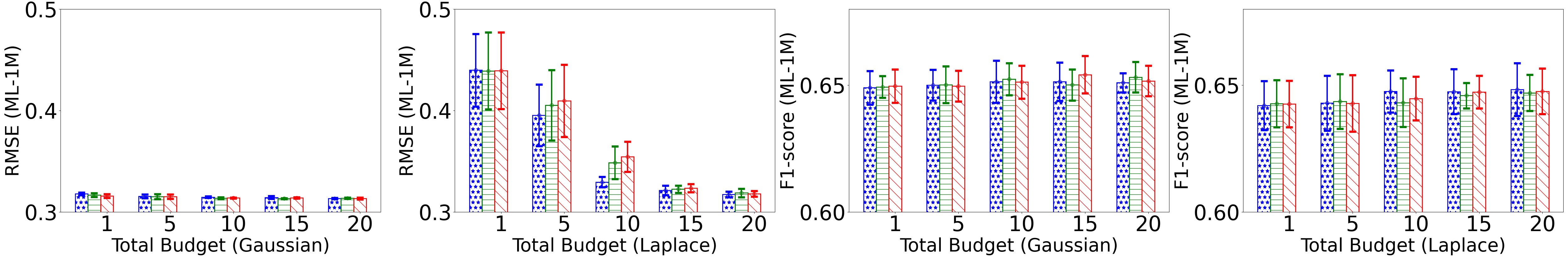}
\label{fig:T_min_ML-1M_all}
    }
\caption{The sensitivity experiments of BGTplanner with $T_{min}$, which controls the range of the action space, were conducted across different privacy budget settings and DP noise mechanisms using the Filmtrust, ML-100K, and ML-1M datasets, respectively. Lower RMSE values indicate better performance, while higher F1-scores reflect better performance.}
\label{fig:hyper_T_min}
\end{figure*}

\subsubsection{System Settings}{\color{black}
Referring to \cite{DBLP:conf/ijcai/YangZHWS23}, we set the hyper-parameter of GPR the length scale parameter $s = 0.2$ and $\alpha=0.001$ for the squared exponential kernel. For CMAB~\cite{DBLP:conf/aistats/HanZWXZ23}, the number of arms is set to $A=5$, the number of initial rounds is set to $T_{0} = 5$ and the trade-off parameter is set to $\gamma=2 \cdot \sqrt{\frac{A \cdot T}{(U + 2) \cdot log(T)}}$ in Eq.~\eqref{eq: probability_action_a} according to~\cite{DBLP:conf/aistats/HanZWXZ23}.
By default, we set $T=100$ while the overall privacy budget $\epsilon_u^{total}$ is set to 10 and $\delta_u^{total}=e^{-5}$~\cite{wu2021fedgnn} for all clients in all datasets.
Additionally, the mapping function $C_u(\cdot)$ uniformly distributes privacy budget values within the range $[\frac{\epsilon_u^{total}}{T},\frac{\epsilon_u^{total}}{T_{min}}]$, where $T_{min}$ is default as $70$ to control the maximum budget cost per round.
We also conduct sensitivity experiments for the hyper-parameters $\epsilon_u^{total}$, $T$, and $T_{min}$ to further investigate their impact on the performance.
}

We adopt the Root Mean Square Error (RMSE) and F1-score as evaluation metrics, both of which are widely used for assessing recommenders. For the F1-score, we binarize the ratings and the outputs from the recommenders to calculate rating classification precision and recall, which are then combined to compute the F1-score. The binarization threshold is set to $0.5$. The reward in round $t$ is quantified by $r^{t}=\text{RMSE}^{t-1}-\text{RMSE}^{t}$.
All experiments are tested by Python and executed on an Intel\textsuperscript{®} CPU server with a Xeon\textsuperscript{®} E5-2660 v4 CPU and 189 GB of RAM.

\subsubsection{Baselines}
To demonstrate the superiority of \emph{BGTplanner}, we compare it with the following baselines.
(1) \textbf{FedSGD}~\cite{mcmahan2017}, the fundamental FL algorithm without DP noises, serves as the benchmark to establish the upper bound accuracy of DP-based baselines.
(2) \textbf{ADPML}~\cite{electronics12030658} determines the minimum privacy budget ($\epsilon_{min}$) and the maximum privacy budget ($\epsilon_{max}$) before training. It gradually increases the privacy budget used by clients during training. We set $\epsilon_{min}= 1$, $\epsilon_{max} =10$ and the increasing rate as $0.9$ in all datasets following the defaulted settings in~\cite{electronics12030658}.
(3) \textbf{ADAP}~\cite{DBLP:conf/trustcom/FuCH22} dynamically adjusts the clip threshold during training and gradually increases the privacy budget used by clients based on the trend of the test loss in model performance. All configurations for this baseline follow the defaults specified in~\cite{DBLP:conf/trustcom/FuCH22}.
(4) \textbf{DPSGD}~\cite{DBLP:conf/ccs/AbadiCGMMT016}, designed a more advanced Gaussian mechanism (GM). It uses the Moments Accountant (essentially RDP) to accumulate and allocate privacy loss for the GM. We utilize the open source code with the same settings from~\cite{Yang2022} for our experiments.

\begin{figure*}[t] 
\centering
    \subfigure{  \includegraphics[width=0.985\linewidth]{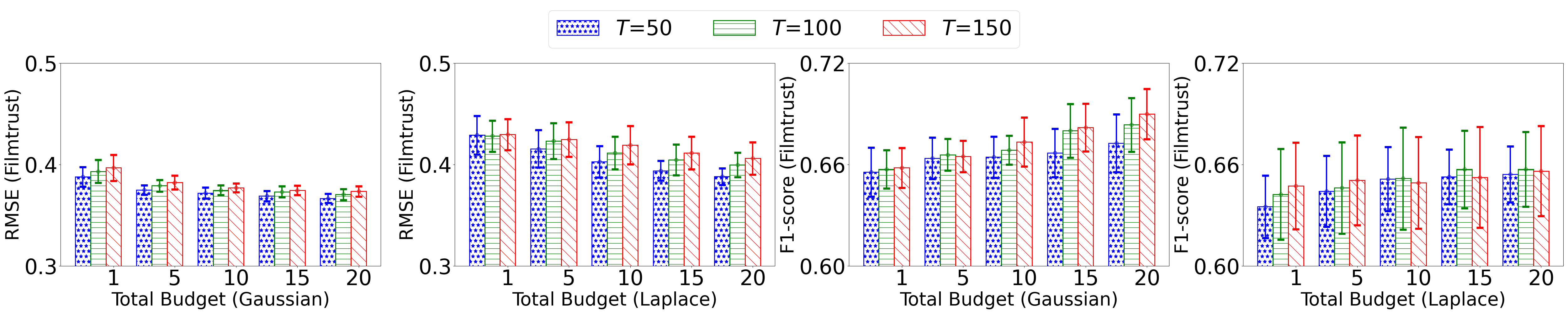}
    \label{fig:T_total_Filmtrust_all}
    }
    \subfigure{     
\includegraphics[width=0.985\linewidth]{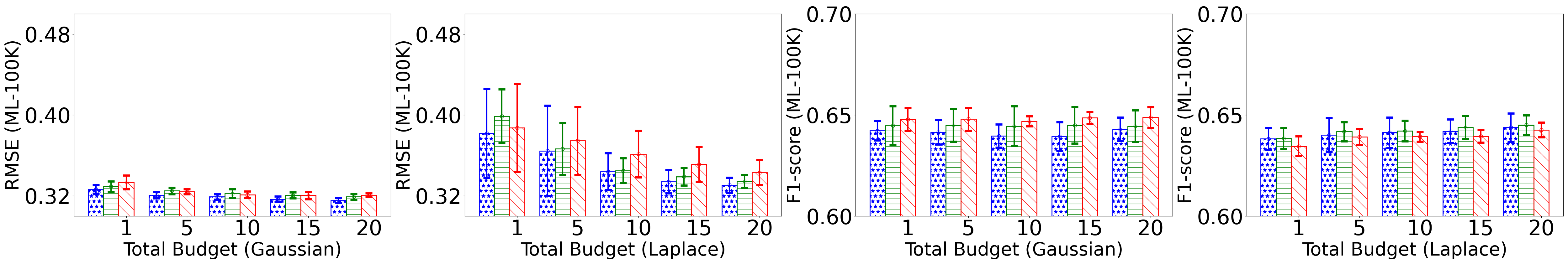}
    }
    \label{fig:T_total_ML-100K_all}
    \subfigure{     \includegraphics[width=0.985\linewidth]{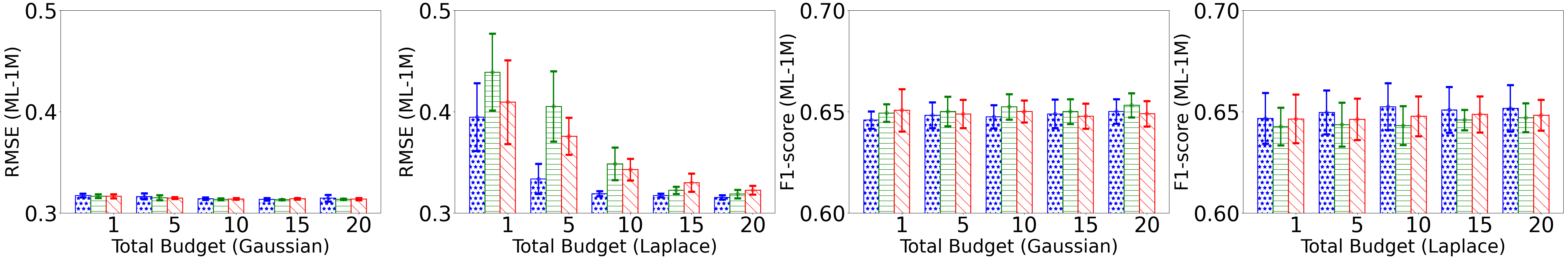}
\label{fig:T_total_ML-1M_all}
    }
\caption{The sensitivity experiments of BGTplanner with $T$, which represents the maximum number of training rounds, are conducted across different privacy budget settings and DP noise mechanisms using the Filmtrust, ML-100K, and ML-1M datasets, respectively. Lower RMSE values indicate better performance, while higher F1-scores reflect better performance.}
\label{fig:hyper_T_total}
\end{figure*}

\subsection{Recommendation Performance Results}
Table~\ref{table:performance} 
shows the recommendation performance between BGTplanner and other baselines under both the Gaussian mechanism (GM) and Laplace mechanism (LM) with $\epsilon_u^{total}=10$.
\textcolor{black}{
In Fig.~\ref{fig:all_baselines}}, 
we further compare final recommendation performance together with standard deviation constrained by various total privacy budgets. 
From Table~\ref{table:performance} and \textcolor{black}{Fig.~\ref{fig:all_baselines}
}, 
we have the following 
observations:
\begin{enumerate}
    \item BGTplanner consistently achieves the recommendation performance closest to FedSGD (which does not incorporate DP noises for gradient protection) across all experimental cases.  In Table~\ref{table:performance}, BGTplanner improves RMSE by an average of 3.40\% with the GM and 10.12\% with the LM. Additionally, in terms of the F1-score, BGTplanner outperforms the second-best baseline by an average of 0.71\% across all datasets and noise mechanisms.
    \item As shown in \textcolor{black}{
Figs.~\ref{fig:all_baselines}}, the performance gain diminishes as the total privacy budget increases, due to the noise variance approaching the increased privacy budget, thereby shrinking the improvement space achievable by tuning privacy budget consumption. Meanwhile, the results demonstrate better training performance of BGTplanner in strict budget constraints, indicating that BGTplanner can allocate the privacy budget strategically with the CMAB method and the long-term budget Constrainter.
    \item DPSGD achieves the second-best performance among DP-based algorithms in most cases when using the GM.
    This is because DPSGD, using RDP, can achieve a tighter bound to track privacy loss during training. However, DPSGD's uniform noise variance allocation across iterations fails to minimize noise impact in the training process, allowing the superior performance of BGTplanner.
    \item ADPML and ADAP rely on heuristics or prior knowledge to set privacy budgets and adjust hyper-parameters. These constraints are especially problematic when dealing with different privacy requirements and make it difficult for these methods to adapt to the dynamic nature of changing settings, which may lead to worse results.
\end{enumerate}


\subsection{Sensitive Experiments of Hyperparameters}
Additionally, we conducted a series of experiments to analyze the sensitivity of the BGTplanner algorithm to key hyperparameters, specifically \( T_{min} \) and \( T \), which control the range of action space and the maximum number of training rounds, respectively. The results of these experiments are presented in Figs.~\ref{fig:hyper_T_min}-\ref{fig:hyper_T_total},  where all other parameters are held constant while varying $\epsilon_u^{total}$ and $T_{min}$ (or $T$), respectively.

\subsubsection{Overall Observations}
In general, the experiments consistently show that as the total privacy budget increases, the model's training performance improves. Moreover, the robustness of BGTplanner is evident from the relatively small variations in performance across different settings of \( T_{min} \) and \( T \), compared to baseline methods. This stability underscores the algorithm's superiority in various training scenarios.
Additionally, as illustrated in Figs.~\ref{fig:hyper_T_min}-\ref{fig:hyper_T_total}, BGTplanner enhances overall training efficiency as the total privacy cost increases for both Gaussian and Laplace mechanisms. However, this efficiency comes at the expense of greater privacy disclosure, highlighting the need to balance privacy and performance. 
%

\subsubsection{Sensitivity to \( T_{min} \)}
Figs.~\ref{fig:hyper_T_min} explores the impact of varying \( T_{min} \), which represents the minimum number of training rounds. A smaller \( T_{min} \) implies a broader range of possible actions for the CMAB model, allowing it to allocate a larger portion of the privacy budget to each training round. Generally, we observe that reducing \( T_{min} \) tends to improve the performance of privacy budget allocation, as it provides more flexibility for distributing the budget effectively across rounds. However, in scenarios where the total privacy budget is limited, an excessively small \( T_{min} \) may lead to rapid consumption of the privacy budget, causing premature termination of training and resulting in suboptimal performance.

\subsubsection{Sensitivity to \( T \)}
Figs.~\ref{fig:hyper_T_total} investigates the influence of the maximum number of training rounds \( T \) on the algorithm's performance. The results reveal distinct trends between the Gaussian and Laplace mechanisms. Under the Gaussian mechanism, BGTplanner achieves better training efficiency with a larger \( T \) ia large proportion of situations, owing to the advanced RDP method, which allows tighter tracking of privacy loss as the number of rounds increases. Conversely, under the Laplace mechanism, increasing \( T \) tends to reduce training efficiency. This discrepancy likely arises because the basic composition theorem, used in the Laplace mechanism, is more sensitive to changes in \( T \), necessitating more careful budget allocation.

Other than the experiments presented above, we also examine the sensitivity of some crucial parameters 
in Appendix~F, 
which illustrates how the hyperparameters affect the overall allocation performance. 
Overall, the results underscore the effectiveness of BGTplanner in achieving high model accuracy under conditions, demonstrating its superiority over other baseline methods.


\section{Conclusion}\label{sec:conclusion}
In this study, we devise a novel online privacy budget allocation solution, \emph{i.e.}, BGTplanner, for differentially private federated recommenders (DPFRs) under dynamic contexts,  solving the challenging problem of privacy budget allocation in federated recommender systems. 
We have pioneered the representation of privacy budget allocation in federated recommender (FR) as a budget-constrained online problem, 
making 
a novel contribution in this field. Furthermore,  Gaussian Process Regression
(GPR) is employed to predict model learning progress. A Contextual Multi-Armed Bandit (CMAB) algorithm is devised to make final budget allocation decisions 
on the fly 
with regret guarantees. 
The empirical evidence from the extensive experiments demonstrates the superiority of our method over existing privacy budget allocation algorithms. To solidify and extend our findings, more comprehensive and detailed experimental analysis, and the customization of privacy budgets for individual clients deserve exploration in future research.



\bibliographystyle{IEEEtran}
\bibliography{ref}



\begin{IEEEbiography}
[{\includegraphics[width=1in,height=1.25in,clip,keepaspectratio]{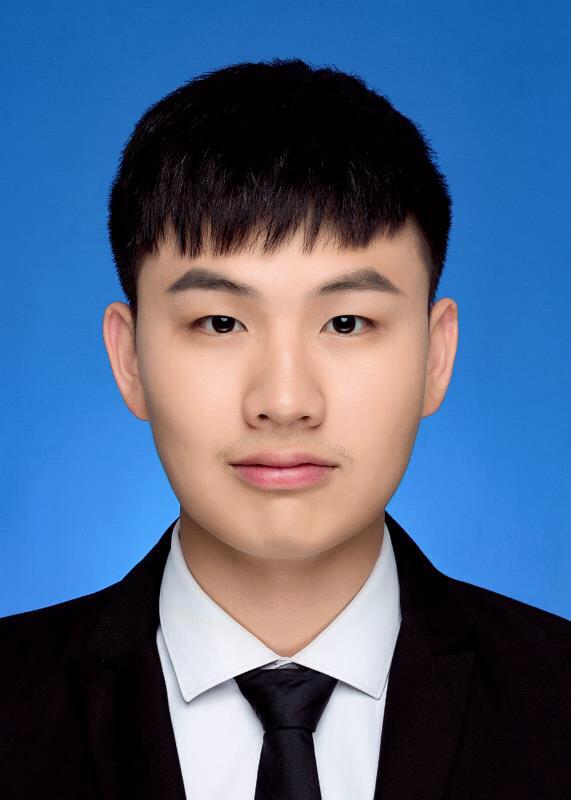}}]
{Xianzhi Zhang} received his B.S. degree from Nanchang University (NCU), Nanchang, China, in 2019 and an M.S. degree from the School of Computer Science and Engineering, Sun Yat-sen University (SYSU), Guangzhou, China, in 2022.  He is currently pursuing a Ph.D. degree in the School of Computer Science and Engineering at Sun Yat-sen University, Guangzhou, China. He is also working as a visiting PhD student at the School of Computing,  Macquarie University, Sydney, Australia. Xianzhi's current research interests include video caching, privacy protection, federated learning, and edge computing. His researches have been published at IEEE TPDS and TSC, and won the Best Paper Award at PDCAT 2021. 
\end{IEEEbiography}

\begin{IEEEbiography}[{\includegraphics[width=1in,height=1.25in,clip,keepaspectratio]{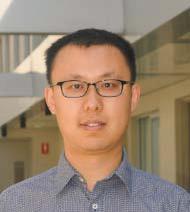}}]
{Yipeng Zhou} is a senior lecturer in computer science with School of Computing at Macquarie University, and the recipient of ARC DECRA in 2018. From Aug. 2016 to Feb. 2018, he was a research fellow with Institute for Telecommunications Research (ITR) of University of South Australia. From 2013.9-2016.9, He was a lecturer with College of Computer Science and Software Engineering, Shenzhen University. He was a Postdoctoral Fellow with Institute of Network Coding (INC) of The Chinese University of Hong Kong (CUHK) from Aug. 2012 to Aug. 2013. He won his PhD degree and Mphil degree from Informatio Engineering (IE) Department of CUHK respectively. He got Bachelor degree in Computer Science from University of Science and Technology of China (USTC). His research interests lie in federated learning, privacy protection and caching algorithm design in networks. He has published more than 80 papers including IEEE INFOCOM, ICNP, IWQoS, IEEE ToN, JSAC, TPDS, TMC, TMM, etc.
\end{IEEEbiography}

\begin{IEEEbiography}[{\includegraphics[width=1in,height=1.25in,clip,keepaspectratio]{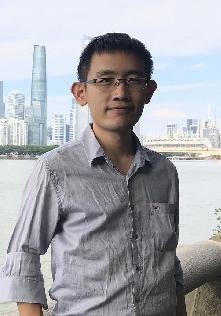}}]
{Miao Hu} is currently an Associate Professor with the School of Computer Science and Engineering, Sun Yat-Sen University, Guangzhou, China. He received the B.S. degree and the Ph.D. degree in communication engineering from Beijing Jiaotong University, Beijing, China, in 2011 and 2017, respectively. From Sept. 2014 to Sept. 2015, he was a Visiting Scholar with the Pennsylvania State University, PA, USA. His research interests include edge/cloud computing, multimedia communication and software defined networks.
\end{IEEEbiography}

\begin{IEEEbiography}[{\includegraphics[width=1in,height=1.25in,clip,keepaspectratio]{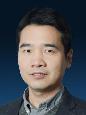}}]
{Di Wu} received the B.S. degree from the University of Science and Technology of China, Hefei, China, in 2000, the M.S. degree from the Institute of Computing Technology, Chinese Academy of Sciences, Beijing, China, in 2003, and the Ph.D. degree in computer science and engineering from the Chinese University of Hong Kong, Hong Kong, in 2007. He was a Post-Doctoral Researcher with the Department of Computer Science and Engineering, Polytechnic Institute of New York University, Brooklyn, NY, USA, from 2007 to 2009, advised by Prof. K. W. Ross. Dr. Wu is currently a Professor and the Associate Dean of the School of Computer Science and Engineering with Sun Yat-sen University, Guangzhou, China. He was the recipient of the IEEE INFOCOM 2009 Best Paper Award, IEEE Jack Neubauer Memorial Award, and etc. His research interests include edge/cloud computing, multimedia communication, Internet measurement, and network security.
\end{IEEEbiography}

\vspace{-5mm}
\begin{IEEEbiography}
[{\includegraphics[width=1in,height=1.25in,clip,keepaspectratio]{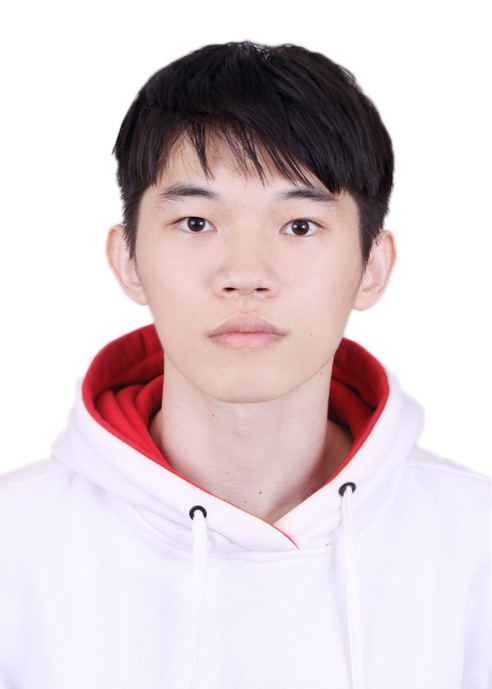}}]
{Pengshan Liao} received the BE degree and the MS degree from Sun Yat-sen University (SYSU), Guangzhou, China. His research interests include federated learning and recommendation system.
\end{IEEEbiography}

\begin{IEEEbiography}[{\includegraphics[width=1in,height=1.25in,clip,keepaspectratio]{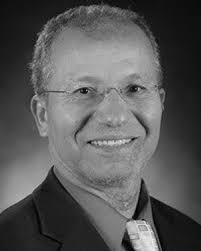}}]
{Mohsen Guizani} received the B.S. (Hons.) and first M.S. degrees in electrical engineering, and the M.S. and Ph.D. degrees in computer engineering from Syracuse University, Syracuse, NY, USA, in 1984, 1986, 1987, and 1990, respectively. He is currently a Professor with the the Machine Learning Department, Mohamed Bin Zayed University of Artificial Intelligence and the Computer Science and Engineering Department, Qatar University, Qatar. Previously, he has served in different academic and administrative positions with the University of Idaho, Western Michigan University, the University of West Florida, the University of MissouriKansas City, the University of Colorado Boulder, and Syracuse University. He is the author of nine books and more than 600 publications in refereed journals and conferences. His research interests include wireless communications and mobile computing, computer networks, mobile cloud computing, security, and smart grid. Throughout his career, he received three teaching awards and four research awards. He was a recipient of the 2017 IEEE Communications Society Wireless Technical Committee Recognition Award, the 2018 Ad Hoc Technical Committee Recognition Award for his contribution to outstanding research in wireless communications and Ad-Hoc Sensor networks, and the 2019 IEEE Communications and Information Security Technical Recognition Award for outstanding contributions to the technological advancement of security. He guests edited a number of special issues in IEEE journals and magazines. He is also the Editor-in-Chief of IEEE Network Magazine. He serves on the editorial boards for several international technical journals, and the Founder and the Editor-in-Chief for Wireless Communications and Mobile Computing (Wiley). He also served as a member, the chair, and the general chair of a number of international conferences. He was the Chair of IEEE Communications Society Wireless Technical Committee and the Chair of the TAOS Technical Committee. He has served as the IEEE Computer Society Distinguished Speaker. He is currently an IEEE ComSoc Distinguished Lecturer. He is a Senior Member of ACM.
\end{IEEEbiography}

\begin{IEEEbiography}[{\includegraphics[width=1in,height=1.25in,clip,keepaspectratio]{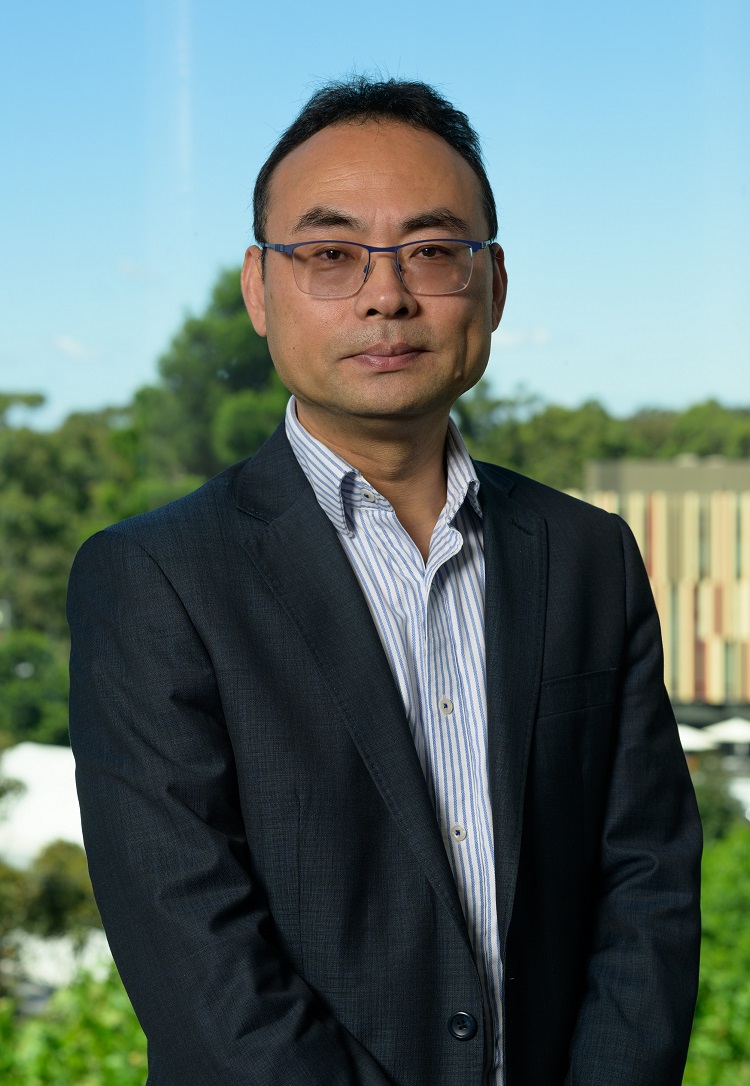}}]
{Michael Sheng} is a Distinguished Professor and Head of School of Computing at Macquarie University. Before moving to Macquarie, Michael spent 10 years at School of Computer Science, the University of Adelaide, serving in senior leadership roles such as Interim Head of School and Deputy Head of School. Michael holds a PhD degree in computer science from the University of New South Wales (UNSW) and did his post-doc as a research scientist at CSIRO ICT Centre. From 1999 to 2001, Sheng also worked at UNSW as a visiting research fellow. Prior to that, he spent 6 years as a senior software engineer in industries. 

Prof. Michael Sheng's research interests include Web of Things, Internet of Things, Big Data Analytics, Web Science, Service-oriented Computing, Pervasive Computing, and Sensor Networks. He is ranked by Microsoft Academic as one of the Most Impactful Authors in Services Computing (ranked Top 5 all time worldwide) and in Web of Things (ranked Top 20 all time). He is the recipient of the AMiner Most Influential Scholar Award on IoT (2019), ARC Future Fellowship (2014), Chris Wallace Award for Outstanding Research Contribution (2012), and Microsoft Research Fellowship (2003). Prof Michael Sheng is Vice Chair of the Executive Committee of the IEEE Technical Community on Services Computing (IEEE TCSVC), the Associate Director (Smart Technologies) of Macquarie's Smart Green Cities Research Centre, and a member of the ACS Technical Advisory Board on IoT.
\end{IEEEbiography}

\clearpage
\begin{appendices}

\end{appendices}

\vfill

\end{document}